\documentclass{article}
\usepackage[margin=0.7in]{geometry}
\usepackage{amsmath, amssymb, amsthm}
\usepackage{natbib}
\usepackage{algorithm}
 \usepackage{hyperref}
 \usepackage{xurl}
\usepackage{algpseudocode}
\usepackage{enumitem}
\usepackage{graphicx}
\usepackage{tikz}
\usetikzlibrary{decorations.pathreplacing, arrows.meta, calc}

\usepackage{amsfonts}
\usepackage{times}
\usepackage{latexsym}
\usepackage{amssymb}
\usepackage{amsmath}
\usepackage{verbatim}
\newcommand{\R}{\mathbb{R}}

\def\bb0{{\mathbb{0}}}

\def\bb{{\mathbf{b}}}

\def\b0{{\mathbf{0}}}
\def\opt{\mathsf{OPT}}

\def\bX{{\mathbf{X}}}

\def\b1{{\mathbf{1}}}

\def\bbR{{\mathbb{R}}}

\def\cA{\mathcal{A}}
\def\cB{\mathcal{B}}
\def\cC{\mathcal{C}}

\def\cF{\mathcal{F}}
\def\cG{\mathcal{G}}

\def\cR{\mathcal{R}}

\def\cX{\mathcal{X}}

\def\sfT{\mathsf{T}}

\def\sfp{{\mathsf{p}}}

\def\sf0{{\mathsf{0}}}

\def\nn{\nonumber}

\newtheorem{theorem}{Theorem}
\newtheorem{definition}[theorem]{Definition}
\newtheorem{lemma}{Lemma}
\newtheorem{rem}{Remark}

\begin{document}
\title{Convex Optimization with Nested Evolving Feasible Sets}
\author{
Karthick Krishna M., Haricharan Balasundaram, and Rahul Vaze
}
\maketitle              
\begin{abstract}
\emph{Convex Optimization with Nested Evolving Feasible Sets (CONES)} is considered where the objective function \(f\) remains fixed but the feasible region evolves over time as a nested sequence \(S_1 \supseteq S_2 \supseteq \cdots \supseteq S_T\). The goal of an online algorithm is to simultaneously minimize the regret with respect to hindsight static optimal benchmark and the total movement cost $M_\cA(T)$ while ensuring feasibility at all times.
CONES is an optimization-oriented generalization of the well-known \emph{nested convex body chasing} (NCBC). 
When the loss function is convex, we propose a lazy-algorithm and show that it achieves $O(T^{1-\beta}), O(T^\beta)$ simultaneous regret and movement cost for any $\beta \in (0,1]$,  over a time horizon of $T$. 
When the loss function is strongly convex,  we propose a \textsc{Frugal} algorithm that simultaneously achieves zero regret and a movement cost of $O(\log T)$. To complement this, we show that any online algorithm with $o(T)$ regret has a movement cost of $\Omega\left(\sqrt{\frac{\log{T}}{\log \log T}}\right)$.
\end{abstract}
\section{Introduction}

Constrained convex optimization of the form
$\min_{x \in S} f(x),$
is a fundamental paradigm in optimization and learning, capturing a wide range of problems across operations research, control, machine learning, and economics. The formulation is both expressive and tractable: convexity of the objective ensures global optimality, while the feasible set \(S\) encodes system limitations, resource constraints, or structural requirements. As a result, this framework serves as a universal abstraction for modeling decision-making under constraints.

In this paper, we introduce and study \emph{Convex Optimization with Nested Evolving Feasible Sets (CONES)}, where the loss function \(f\) remains fixed but  the feasible region evolves over time as a nested sequence \(S_1 \supseteq S_2 \supseteq \cdots \supseteq S_T\). At each time step \(t\), the decision maker must select a point \(x_t \in S_t\), trading off loss function value $f(x_t)$ and movement across iterates $||x_t-x_{t-1}||$. This model generalizes classical constrained convex optimization by allowing the feasible region to change over time, reflecting progressively revealed or tightening constraints. Beyond its conceptual generality, CONES captures a variety of important applications, which we discuss in detail in Section~\ref{sec:applications}.

To be precise, with CONES,  a convex (loss) function $f: \cX \rightarrow \bbR$,  is known beforehand, and 
at each round $t$,  a  convex set $S_t \subseteq \cX $ is revealed such that 
$S_t\subseteq S_{t-1}$, and  $S_0=\cX \subset \bbR^d$ is a convex, compact, and bounded set. Once $S_t$ is 
revealed, the objective for an online algorithm $\cA$ is to choose action $x_t \in S_t$ (feasible action) so as to {\bf simultaneously} minimize the 
loss function cost 
\begin{equation}\label{defn:optcost}
C_\cA(T) = \sum_{t=1}^T f(x_t),
\end{equation}
and total movement cost 
\begin{equation}\label{defn:mcost}
M_\cA(T) =   \sum_{t=1}^T  ||x_t - x_{t-1}||,
\end{equation}
where 
 $x_0 \in \cX$ is some fixed action. All norms considered in the paper are $2$-norms. Discussion on extension for general norms can be found in Section \ref{sec:norm}.

The performance of $\cA$ is compared against a {\it static} optimal benchmark $\opt$ that chooses its action $x^\opt \in S_T$ that  minimizes the loss function cost, i.e., $x^\opt = \arg  \min_{x\in S_T} f(x)$ and  
$C_\opt(T) = T f(x^\opt)$. Since $S_t$'s are nested, $x^\opt$ is feasible with respect to 
all $S_t's$.\footnote{$T$ is the largest $t$ such that $S_t\ne \emptyset$.} 
The movement cost of $\opt$ is simply 
$M_\opt(T) = ||x^\opt-x_0||.$
The {\bf regret} of $\cA$ is then defined as 
\begin{equation}\label{defn:statregret}
\mathrm{Regret}_\cA(T)=\sup_{f, S_t} \{C_\cA(T) - C_\opt(T)\}.
\end{equation} 
For $\cA$ its two objectives are: minimize $\mathrm{Regret}_\cA(T)$ and movement cost $M_\cA(T)$ simultaneously.\footnote{We are not subtracting movement cost of $\opt$ since it is simply $||x^\opt-x_0||$.} 
The quest is to design an algorithm that achieves zero regret with optimal sub-linear movement cost; failing that, characterize the optimal Pareto-efficient tradeoff between regret and movement cost.

\begin{rem} While the static optimal benchmark is weaker than a dynamic benchmark that allows time-varying actions $x_t^\opt$, we focus on the static setting in this work as a first step in understanding the CONES framework, which itself poses significant challenges. 
\end{rem}

\subsection{Related Work} 
{\bf Nested convex body chasing (NCBC)} CONES with $f\equiv 0$ is known as NCBC, where the only objective is to minimize the movement cost. Starting with \cite{linial}, algorithm design for NCBC has been considered for  minimizing the competitive ratio, the ratio of $M_\cA(T)$ \eqref{defn:mcost} and $M_\opt(T)$ (as defined above),  maximized over the input $\{S_t\}_{t=1}^T$. 

With NCBC, the competitive ratio is known to be $\Omega(\sqrt{d})$ for any online algorithm \cite{linial}. Early progress on deriving upper bounds on competitive ratio for NCBC started with \cite{bansal2017nestedconvexbodieschaseable} that proposed an algorithm with competitive ratio that is exponential in $d$. A greedy algorithm that simply chooses the nearest point on the newly revealed set $S_t$ has a competitive ratio of $O(d^{d/2})$ \cite{argue2018nearlylinearboundchasingnested}. 
Thereafter in major breakthroughs, recursive {\it centroid} algorithm \cite{argue2018nearlylinearboundchasingnested} was shown to have a  competitive ratio of $O(d\log d)$, which was improved in \cite{bubeck2021chasingnestedconvexbodies} to $O(\min\{d, \sqrt{d\log T}\})$ using an algorithm that moves to the {\it Steiner} points of $S_t$'s.
 The intuition behind these algorithms is to move to a point well inside the revealed body so that the future movement cost is large  only if the sets revealed in the future  shrink sufficiently. 
The non-nested version of {\it convex body chasing} (CBC) has also been studied with the best known bounds on competitive ratio being $O(\min\{d, \sqrt{d\log T}\})$ \cite{guptatang, Sellke2023}.

{\bf Online Convex Optimization with Movement Cost} The joint convex loss function and movement cost minimization has been studied in prior work
by taking a linear combination of them. In particular, let at time $t$, a convex function $f_t$ ($f_t$ is allowed to change with time) is revealed and an online algorithm $\cA$ can choose its action  $x_t$ at time $t$ knowing $f_1, \dots, f_t$ so as to minimize 
\begin{equation}\label{eq:lincost}
\sum_{t=1}^T f_t(x_t) + ||x_t-x_{t-1}||^\sfp
\end{equation}
where both $\sfp=1,2$ have been well-studied. The performance metric for \eqref{eq:lincost} is the competitive ratio.
Most of the prior work \cite{bansal20152, Sellke2023, chen2018smoothed, goel2019beyond, zhang2021revisiting} on deriving competitive ratio guarantees for problem \eqref{eq:lincost} allowed $x_t\in \bbR^d$, except \cite{argue2020dimension}, where $x_t\in S\subset \bbR^d$, but $S$ remains fixed for all $t=1,\dots, T$. 

The analogue of \eqref{eq:lincost} for our CONES  setting is to solve \eqref{eq:lincost} under the constraint that $x_t\in S_t$, where $S_t$'s are  (nested) convex sets that are changing over time, which we refer to as {\it dynamic constraints}. 

{\bf Constrained Online Convex Optimization (COCO)}
With COCO, an algorithm $\cA$ has to choose actions $x_t$ at time $t$ {\bf before} a convex loss function $f_t$ and a convex set $S_t, S_t\subseteq S_{t-1}$ are revealed. With COCO, once $x_t$ has been chosen, an adversary reveals $f_t$ and $S_t$, and the objective for an online algorithm $\cA$ is to simultaneously minimize the  regret 
$\mathrm{Regret}^{COCO}_\cA(T)=\sup_{f_t, S_t} \left\{\sum_t f_t(x_t) - \min_{x\in S_T}\sum_t f_t(x)\right\},$
and total constraint violation 
$\mathrm{Violation}_\cA(T)=\sup_{f_t, S_t} \left\{\sum_t \text{dist}(x_t, S_t)\right\},$
 where $\text{dist}(x_t, S_t)$ is the distance between $x_t$ and $S_t$, and $\opt$ chooses a single action $x^\star\in S_T$ that minimizes $\sum_t f_t(x)$ with no violation, i.e. $x^\star\in S_T$.
 
COCO has tradeoffs similar  to CONES since minimizing $\mathrm{Regret}^{COCO}_\cA(T)$ and $\mathrm{Violation}_\cA(T)$ are at odds with each other. COCO is a very well-studied object \cite{yu2017online, pmlr-v70-sun17a, yi2023distributed, neely2017online, georgios-cautious, guo2022online, Sinha2024,Vaze2025b}.
The current best known results \cite{Sinha2024} on COCO  achieve simultaneous $\mathrm{Regret}^{COCO}_\cA(T)= O(\sqrt{T})$  and $\mathrm{Violation}_\cA(T) =O(\sqrt{T})$ using an algorithm that follows online gradient descent on a weighted combination of the most recent loss function  and constraint violation. For special cases of $S_t$'s, e.g. when $S_t$'s are spheres or axis-aligned polygons, $\mathrm{Regret}^{COCO}_\cA(T)=O(\sqrt{T})$  and $\mathrm{Violation}_\cA(T)=O(1)$ have been shown to be achievable in \cite{Vaze2025a} where the algorithm takes a gradient descent (GD) step with respect to the most recently revealed function $f_{t-1}$ followed by a projection on to the most recent constraint set $S_{t-1}$.   The only known lower bound \cite{Sinha2024} on the simultaneous  regret and total constraint violation for COCO is 
$\Omega(\sqrt{d}), \Omega(\sqrt{d})$, respectively,

 COCO and  CONES, however, are principally different theoretical objects because of 
 i) information structure, i.e. when $f$ and $S_t$'s are revealed, ii) feasibility constraint, and iii) since movement cost and total constraint violation are fundamentally different. 
 
 CONES is  also similar to some of the classical optimization paradigms, e.g. perturbation analysis of optimization problems  \cite{bonnans2000perturbation}, that we discuss in detail in Section \ref{sec:classicalpriorwork}.
\subsection{Our Contributions}
\begin{enumerate}[leftmargin=*,nosep]
\item When $f$ is {\bf strongly convex}, we show that the greedy algorithm $\cG$ that chooses $x_t=x_t^\star$ for all $t$ has minimum regret among all online algorithms but incurs a movement cost of $\Theta(\sqrt{T})$. Next, we 
propose a more `lazy' algorithm compared to $\cG$ called \textsc{Frugal}  that plays action  $x_t= x_t^\star$ only when it is forced to do so in order to keep its regret non-positive and otherwise plays the action obtained by projecting its current action $x_{t-1}$
on to the most recently revealed set $S_t$. We show that the regret of \textsc{Frugal} is non-positive and its movement cost is $O(\log T)$, a fundamental improvement over $\cG$ when $f$ is also smooth. To complement this upper bound, we also show that that any online algorithm with non-positive or sub-linear regret has movement cost $\Omega\left(\sqrt{\frac{\log{T}}{\log \log T}}\right)$.


\item When $f$ is {\bf convex}, we show a negative result that $\cG$ has a movement cost of $\Omega(T)$. Next, we propose a more conservative algorithm than $\cG$ that achieves a (regret, movement cost) tuple of $(T^{1-\beta}, T^{\beta})$ for any $\beta \in [0,1]$. Finally, for a minimally structured input, we show that a modified \textsc{Frugal} algorithm achieves $O(1)$ regret and $O(\log^2 T)$ movement cost.


\item For solving problem \eqref{eq:lincost} with dynamic constraints when $f_t$'s are convex, we derive a simple algorithm with competitive ratio of $5$ for $d=1$ in  Section \ref{sec:5comp}. When $f$ is $\alpha$-sharp, we show that the greedy algorithm $\cG$ achieves  competitive ratio of $\frac{12}{\alpha}$ for all $d$.
\end{enumerate}

\vspace{-0.2in}
\section{Assumptions and Notations}
\vspace{-0.1in}
We use $\Pi_{S}(x)$ to denote the projection of point $x$ on to convex set $S$. $\cB(x, r)$ denotes a ball of radius $r$ centered at $x$. Without loss of generality, we let $f(x)\ge 0, \ \forall \ x\in \cX$.
$\mathbf{X}_t^* := \arg\min_{x\in S_t}f(x)$, $v_t = \min_{x\in S_t}f(x)$. When the optimizer is unique, $x_t^\star= \arg \min_{x\in S_t} f(x)$ and $v_t=f(x_t^\star)$.
\textbf{Assumption 1 (Convexity)} 
$\mathcal{X} \subset \mathbb{R}^d$ is the admissible set that is closed, convex and has a finite Euclidean diameter $D$, and $S_1 \subseteq \cX$. The cost function $f : \mathcal{X} \mapsto \mathbb{R}$ is convex.
\textbf{Assumption 2 (Lipschitzness)}
The cost function $f$ is assumed to be $G$-Lipschitz on $\cX$.
\section{Basic Tradeoff with CONES}
With CONES, minimizing regret \eqref{defn:statregret} and movement cost \eqref{defn:mcost} are inherently at odds with each other. To see this clearly, 
as soon as set $S_t$ is revealed at time $t$, the action that minimizes regret is to choose $x_t^\star = \arg \min_{x\in S_t} f(x)$ but that incurs a cost of $||x_t^\star - x_{t-1}||$, which could be large. Alternatively, an algorithm can move `lazily' to the nearest point in $S_t$, i.e. $x_t = \Pi_{S_t}(x_{t-1})$ but that does not allow any control over minimizing regret. Thus, an optimal algorithm for CONES has to find a right balance between playing sub-optimal actions and incurring large movement cost. 
In the rest of the paper, we study CONES \eqref{defn:mcost} and \eqref{defn:statregret} when $f$ is strongly convex, $\alpha$-sharp and just convex.

\section{$f$ is Strongly  Convex}\label{sec:strongconvex}
We begin by characterizing the exact movement cost of the greedy algorithm $\cG$ that chooses $x_t=x_t^\star$ for all $t$. 
By its very definition, the greedy algorithm $\cG$ has the least  regret (also non-positive) formally written as follows.
\begin{theorem}\label{lemma:neg_reg}
    $\mathrm{Regret}_\cG(T)\le \mathrm{Regret}_\cA(T) $ for any $\cA$. Moreover, $\mathrm{Regret}_\cG(T)\le 0$.
\end{theorem}

\begin{theorem}\label{thm:lbGreedy} {\bf Lower Bound on the Movement Cost of $\cG$:}
   With strongly convex $f$,  there exists an input instance for which $M_\cG(T) = \Omega(\sqrt{T})$.
\end{theorem}






\begin{theorem}\label{lem:ubG}
{\bf Upper Bound on the Movement Cost of $\cG$:}
    $M_\cG(T) = O(\sqrt{T})$ for all input instances when $f$ is strongly convex.
\end{theorem}
The proof of Theorem \ref{thm:lbGreedy} is provided in Section \ref{app-lbstrconvxGreedy}, while that of Theorem \ref{lem:ubG} is provided in Section \ref{app:ubstrconvxG}. 
Combining Theorem \ref{thm:lbGreedy} and Lemma \ref{lem:ubG}, the exact movement cost of the greedy algorithm $\cG$ for CONES is $\Theta(\sqrt{T})$, when $f$ is strongly convex. The greedy algorithm $\cG$ by definition has the smallest regret among all possible algorithms, and its regret is  typically significantly negative when its movement cost is $\Theta(\sqrt{T})$. Thus, a natural question at this point is: does there exist an algorithm $\cA$ with regret at most $0$ but with  movement cost of $o(\sqrt{T})$. Next, we answer this in the affirmative by defining algorithm \textsc{Frugal}  that has non-positive regret and movement cost of $O(\log T)$.
\begin{algorithm}[H]
\caption{\textsc{Frugal}}
\begin{algorithmic}[1]
\Require Nested sets $S_1 \supseteq S_2 \supseteq \dots \supseteq S_T$, convex $f$.
\Ensure Action sequence $x_1, \dots, x_T$.
\State $F_0 \gets 0$
\For{$t = 1$ \textbf{to} $T$}
    \State Observe $S_t$
    \State $\mathbf{X}_t^* := \arg\min_{x\in S_t}f(x)$, $v_t \gets \min_{x\in S_t}f(x)$
    \State $\hat{x}_t \gets \Pi_{S_t}(x_{t-1})$ \Comment{Euclidean projection}
     \State $x_{\text{near}, t}^\star \gets \Pi_{\mathbf{X}_t^*}(x_{t-1})$  \Comment{Nearest minimizer}
    \If{$F_{t-1} + f(\hat{x}_t) \le t \cdot v_t$}
        \State $x_t \gets \hat{x}_t$ \Comment{Lazy play}
    \Else
        \State $x_t \gets x_{\text{near}, t}^\star$ \Comment{Forced jump}
    \EndIf
    \State $F_t \gets F_{t-1} + f(x_t)$
\EndFor
\end{algorithmic}
\end{algorithm}
The idea behind \textsc{Frugal} is as follows.
For each time $t\le T$, $C_\opt(T) \ge t \cdot v_t$.
If the current action $x_{t-1}$ is not feasible with respect to the newly revealed set $S_t$, \textsc{Frugal} moves {\it lazily} to the closest point ${\hat x}_t$ (projection) on $S_t$ to ensure feasibility as long as moving there does not make its cumulative regret until time $t$ with respect to $x_t^\star$, positive, i.e. its cumulative cost till time $t$ is at most $t \cdot f(x_t^\star)$, the current lower bound on the cost of the $\opt$. 
If that is not possible, i.e. moving to ${\hat x}_t$ will make the regret positive at time $t$,  it moves to the (closest) local minimizer $x_t^\star$ similar to $\cG$, and accumulates negative regret which can be used in future to save on the movement cost. Thus  \textsc{Frugal} tries to avoid moving to the local minimizers compulsively since that 
can force it to have a large movement cost similar to $\cG$, and does so only when forced so as to keep its regret at most $0$. The computation complexity of \textsc{Frugal} is discussed in Section \ref{sec:comp}.

\begin{rem}\label{rem:budget} Recall that the regret of any online algorithm $\cA$  satisfies
\begin{align*}
\mathrm{Regret}_\cA(1) & = (f(x_1) - v_1), \\
\mathrm{Regret}_\cA(2) & = (f(x_1) + f(x_2) - 2 v_2) = \mathrm{Regret}_\cA(1) + f(x_2) -  v_2 -  (v_2 - v_1),\\
& \vdots \\
\mathrm{Regret}_\cA(t) & = \sum_{\ell=1}^{t-1}\mathrm{Regret}_\cA(\ell) + (f(x_t) -  v_t) - (t-1)(v_t -  v_{t-1}).
\end{align*}
We define the final term $(t-1)(v_t -  v_{t-1})$ as the {\bf budget} of $\cA$ at time $t$.
If an algorithm $\cA$ has ensured $\mathrm{Regret}_\cA(\tau)\le 0$ for $\tau \le t-1$, then the {\it budget} at time $t$ captures the 
total allowance $\cA$ has to choose its action $x_t$ such that it can maintain $\mathrm{Regret}_\cA(t)\le 0$. 
At each time $t$, algorithm \textsc{Frugal} uses all its {\it budget} to play only `lazy' actions.
%
\end{rem}

The following is immediate, whose proof we provide in Section \ref{app:lem:regretBalance} for completeness.
\begin{lemma}\label{lem:regretBalance}  $\mathrm{Regret}_{\textsc{Frugal}}(t)\le 0$ for all $t\le T$.
\end{lemma}

Next, we work towards upper bounding the movement cost of algorithm \textsc{Frugal}.

\begin{definition}\label{defn:balancejump} Algorithm \textsc{Frugal} is defined to play a {\bf lazy} action at time $t$ if it selects $x_t= {\hat x}_t$. If it does not play a {\it lazy} action at time $t$, it is defined to {\bf jump} and $x_t= x_{\text{near}, t}^\star$.
\end{definition}

\begin{definition}\label{defn:balancejumptimes}
Let $1 \le t_1 < t_2 < \cdots < t_N \le T$ be the time slots at which \textsc{Frugal}
{\it jumps}, i.e., at time $t_k$, $x_{t_k}=x_{\text{near}, t_k}^\star$. We define {\bf phase} $k$ as the contiguous set of time slot $t_k, t_k+1, \dots, t_{k+1}-1$, where $t_k$ is the {\bf head} of phase $k$.
\end{definition}

The computational complexity and oracle model needed to implement all algorithms in the paper are discussed in Section \ref{sec:comp}.

\subsection{Upper Bound on the Movement Cost of \textsc{Frugal}}\label{sec:ubalphasharp}
\begin{theorem}\label{thm:rsb_sc}
For  $G$-Lipschitz, $L$-smooth, and $\mu$-strongly convex $f$,
$$M_{\textsc{Frugal}}(T) = O\left(d^{d/2}\frac{2G+(L+\mu)D}{\mu} \log T\right)=O(\log T).$$
\end{theorem}
One essential ingredient to  prove Theorem \ref{thm:rsb_sc} is a basic result from convex geometry that is described as follows.
Let $K_1, \dots, K_T$ be nested convex sets (i.e., $\bbR^d \supset K_1 \supseteq K_2 \supset K_3 \supseteq \dots \supseteq K_T$). 

\begin{definition}\label{defn:projectioncurve}
If $\sigma_1\in K_1$, and $\sigma_{t+1} = \Pi_{K_{t+1}}(\sigma_t)$, for $t=1, \dots, T-1$. Then the curve 
$${\underline \sigma}= \{(\sigma_1,\sigma_2), (\sigma_2,\sigma_3), \dots, (\sigma_{T-1},\sigma_T)\}$$ is called the projection curve on $K_1, \dots, K_T$.
\end{definition}

\begin{lemma}\label{lem:projdistanceint} Let $\Sigma = \max_{{\underline \sigma}} \sum_{t=1}^{T-1} ||\sigma_t - \sigma_{t+1}||$, then $\Sigma  \le O(d^{d/2} \text{diameter}(K_1)).$
\end{lemma}
The proof of Lemma \ref{lem:projdistanceint} can be found \cite{Vaze2025a} based on classical result of \cite{Manselli} and is included in Section \ref{app:projdistance} for completeness.

The proof of Theorem \ref{thm:rsb_sc} is provided in Section \ref{appthmrsbsc}, where the main idea is to connect the successive {\it jump} times of \textsc{Frugal} via a geometric (increasing) recurrence that is a function of the distance between successive {\it phase heads}. This recurrence is then used  to bound the total distance moved across {\it heads} of each phase by $O(\log T)$.  
Intuition for the recurrence is two fold: i) for each {\it jump} made by \textsc{Frugal}, the {\it budget} (defined in Remark \ref{rem:budget}) available to play {\it lazy} actions increases progressively while maintaining the regret to be at most $0$, 
and ii)  the strong convexity of $f$ implies that function increases sufficiently as we go away from local minimizers. 
Moreover, since within each phase, the \textsc{Frugal} takes successive projections, the movement cost is bounded in each phase using the classical result (Lemma \ref{lem:projdistanceint}).

Theorem \ref{thm:rsb_sc} shows that when the loss function $f$ is strongly convex and smooth, the movement cost of the \textsc{Frugal} is $O(\log T)$ while its regret is at most $0$ by its very definition. 
Thus, compared to the greedy algorithm $\cG$ that has movement cost of $\Theta(\sqrt{T})$, this is an fundamental improvement, and that too without compromising on regret. 
To complement the result of Theorem \ref{thm:rsb_sc}, we next show a matching lower bound on the movement cost of any online algorithm whose regret at time $t$ is at most $0$ or $O(t^\lambda)$ for $0\le \lambda<1$, when $f$ is strongly convex and smooth.

\subsection{Lower Bound on the Movement Cost for Any Algorithm}\label{sec:lbstrcvx}

In this section, we show that when $f$ is strongly convex, the movement cost of any online algorithm is $\Omega(\log T)$ that maintains a sublinear regret at all times. Formally, the result is stated as follows.
\begin{theorem}\label{thm:lbsclogT}
Let an online algorithm $\cA$ choose $x_t \in S_t$ with
$S_t \subseteq S_{t-1} \subseteq \mathcal{X}$, and maintain the regret
constraint
$\quad \mathrm{Regret}_\cA(t)= \sum_{\tau=1}^t\!\bigl(f(x_\tau) - v_t\bigr) \le c_R t^\lambda$
for constants $c_R > 0$ and $\lambda \in [0,1)$ for all $t \ge 1$,
where $v_t = \min_{x \in S_t} f(x)$.
Then there exists a  strongly convex and smooth $f$, an initial point $x_0 \in \mathcal{X}$ and an adversarial sequence
of nested sets $\{S_t\}$ such that $M_\cA(T) = \Omega\left(\sqrt{\frac{\log{T}}{\log \log T}}\right)$.
\end{theorem}
The proof of Theorem \ref{thm:lbsclogT} is provided in Section \ref{app:thm:lbsclogT},
where, we consider $f(x) = \tfrac{1}{2}\|x\|^2$ and construct $S_t$'s with $d=2$. In particular, set $S_t$'s are constructed such that $x_{t}^\star$ (unique since we are in strongly-convex case) belongs to one of the two opposite sides of a rectangle $\cR$ of $O(1)$ width, for each $t$. Once a {\it distinct} set $S_{t_k}$  is revealed at time $t_k$,  it is is held constant until $\cA$ plays an action belonging to the ball $\cB(x_{t_k}^\star, \epsilon)$ ($\epsilon$ is a constant). Once $\cA$ plays an action belonging to the ball $\cB(x_{t_k}^\star, \epsilon)$, next {\it distinct} set $S_{t_{k+1}}$ is revealed at time $t_{k+1}$ with $x_{t_{k+1}}^\star$ belonging to the opposite side of $\cR$ compared to $x_{t_k}^\star$, and this procedure is repeated. 
In the proof, exploiting the strong convexity of $f(x) = \tfrac{1}{2}\|x\|^2$, we show that for any $\cA$, the time difference between $t_k$ and $t_{k+1}$ can at most grow geometrically with $k$, leading to the $\Omega\left(\sqrt{\frac{\log{T}}{\log \log T}}\right)$ lower bound on the movement cost.

\section{$f$ is Convex}

%
Finally, in this section, we consider the most general case for CONES when $f$ is convex. 
We start by showing that the greedy algorithm $\cG$ has movement cost $\Omega(T)$ when the strong convexity assumption on the cost functions $f_t$ is lifted.

\begin{theorem}\label{thm:lbGconvex}
    When $f$ is convex, $M_\cG(T) = \Omega(T)$.
\end{theorem}

The proof of Theorem \ref{thm:lbGconvex} is provided in Section \ref{app:thm:lbGconvex}.
An obvious  natural question at this stage is: similar to when $f$ is strongly convex or $\alpha$-sharp, can we recover a movement cost  $O(\log T)$ while retaining non-positive regret with \textsc{Frugal} algorithm when $f$ is convex. The question unfortunately is difficult to settle. Moreover, we are unable to derive a better lower bound than Theorem \ref{thm:lbsclogT} that established that the movement cost 
 for any online algorithm is $\Omega(\log T)$ that achieves sub-linear regret at all times when $f$ is strongly convex.
 Thus, next,  we propose an algorithm that achieves (regret, movement cost) tuple of $(T^{\beta}, T^{1-\beta})$ for any $\beta \in [0,1]$ when $f$ is convex.
\vspace{-0.15in}
\subsection{Achieving (regret, movement cost) tuple of $(T^{\beta}, T^{1-\beta})$ for any $\beta \in [0,1]$}
\vspace{-0.1in}
\begin{algorithm}[H]
\caption{Level Set Projection (LSP) Algorithm}
\begin{algorithmic}[1]
\Require Tolerance $\varepsilon > 0$.
\Ensure Action sequence $x_1, \dots, x_T$.
\State $p \gets 1$ \Comment{Phase index}
\State $L_1 \gets \min_{x \in S_1} f(x)$ \Comment{Initial minimum}
\For{$t = 1$ \textbf{to} $T$}
    \State Observe $S_t$
    \State Compute $\Phi_t^{(p)} \gets \{x \in S_t : f(x) \le L_p + \varepsilon\}$
    \If{$\Phi_t^{(p)} = \emptyset$}
        \Comment{Phase transition: minimum has risen above $L_p + \varepsilon$}
        \State $p \gets p + 1$
        \State $L_p \gets \min_{x \in S_t} f(x)$ \Comment{Update level to current minimum}
        \State $\Phi_t^{(p)} \gets \{x \in S_t : f(x) \le L_p + \varepsilon\}$
        \Comment{Redefine active set under new level}
    \EndIf
    \State $x_t \gets \Pi_{\Phi_t^{(p)}}(x_{t-1})$ \Comment{Project onto the active set $\Phi_t^{(p)}$}
\EndFor
\end{algorithmic}
\end{algorithm}
Algorithm \textsc{LSP} maintains a phase index $p$ with an associated level $L_p$ equal to the minimum value of $f$ at the start of the phase. At each time $t$, it defines the active set $
\Phi_t^{(p)} = \{x \in S_t : f(x) \le L_p + \varepsilon\}$
and selects $x_t$ as the Euclidean projection of $x_{t-1}$ onto $\Phi_t^{(p)}$. If $\Phi_t^{(p)}$ is empty, indicating that the current minimum of $f$ over $S_t$ exceeds $L_p + \varepsilon$, the algorithm increments the phase index, updates $L_p = \min_{x\in S_t} f(x)$, and redefines the active set. Thus, the method tracks a sequence of nested $\varepsilon$-sublevel sets, performing minimal-movement updates within a phase and triggering phase transitions only when the target level becomes infeasible.

\vspace{-0.1in}
\begin{theorem}
\label{thm:steiner-convex}
    $ \mathrm{Regret}_{\textsc{LSP}}(T) \le T\varepsilon, \quad \text{and} \quad    M_\textsc{LSP}(T) 
    = O\!\left(\frac{d^{d/2}GD^2}{\varepsilon}\right).$
Choosing $\epsilon = T^{-\beta}$, we get (regret, movement cost) tuple of $(T^{1-\beta}, T^{\beta})$ for any $\beta \in [0,1]$.
\end{theorem}
The proof of Theorem \ref{thm:steiner-convex} is provided in Section \ref{app:thm:steiner-convex}.

\begin{rem}
Replacing the projection step (Line 11 in \textsc{LSP}) $x_t \gets \Pi_{\Phi_t^{(p)}}(x_{t-1})$ by $x_t \gets \textsf{Steiner}\!\left(\Phi_t^{(p)}\right)$ where $\textsf{Steiner}(\cX)$ is the Steiner point \cite{bubeck2021chasingnestedconvexbodies} of set $\cX$, the movement cost 
guarantee can be improved to $O\!\left(\frac{d GD^2}{\varepsilon}\right)$, however, computing Steiner point is very `costly' compared to a projection.
\end{rem}

We present some numerical results in Section \ref{sec:sim} to compare the performance of Greedy, \textsc{Frugal} and TSP algorithm.

\subsection{Achieving $O(1)$ regret and $O(\log ^2T)$ movement cost with Structured Input}
In this section, we consider a special class of input called {\it directional} and show that  \textsc{\textsc{Gap-Frugal}} has constant regret and $O(\log ^2 T)$ movement cost when $f$ is just convex,  a major improvement over the \textsc{LSP} algorithm.

\begin{algorithm}[H]
\caption{\textsc{Gap-Frugal}}
\label{alg:Gap-Balance}
\begin{algorithmic}[1]
\Require Nested sets $S_1\supseteq\cdots\supseteq S_T$,
         $G$-Lipschitz convex $f$.
\State $F_0 \gets 0$, $C \gets 0$ \Comment{Cumulative threshold $C=\sum \epsilon_{\tau}$ where $\epsilon_{\tau} = 1/\tau^2$}
\For{$t=1$ \textbf{to} $T$}
    \State $\epsilon_t \gets 1/t^2$, $C \gets C + \epsilon_t$
    \State $\mathbf{X}_t^* := \arg\min_{x\in S_t}f(x)$, $v_t \gets \min_{x\in S_t}f(x)$
    \State $\hat{x}_t \gets \Pi_{S_t}(x_{t-1})$ \Comment{Euclidean projection}
     \State $x_{\text{near}, t}^\star \gets \Pi_{\mathbf{X}_t^*}(x_{t-1})$ \Comment{Nearest minimizer}
    \If{$F_{t-1} + f(\hat{x}_t) \le t\,v_t + C$ \textbf{or} $f(\hat{x}_t) - v_t \le \epsilon_t$}
        \State $x_t \gets \hat{x}_t$ \Comment{Lazy}
    \Else
        \State $x_t \gets x_{\text{near}, t}^\star$ \Comment{Forced jump}
    \EndIf
    \State $F_t \gets F_{t-1} + f(x_t)$
\EndFor
\end{algorithmic}
\end{algorithm}
Algorithm \textsc{Gap-Frugal} is similar to \textsc{Frugal} and only differs in the conditions in which the {\it lazy} action is taken. In particular, 
the {\it lazy} action is taken 
if either the {\bf cumulative regret condition}
 $\quad  F_{t-1} + f(\hat{x}_t) \le t\,v_t + C \ \ 
$
or the {\bf instantaneous gap condition}
$\quad 
f(\hat{x}_t) - v_t \le \epsilon_t \ \ $
is satisfied. Otherwise, it performs a forced jump to the minimizer $x_t^\star$. Thus, it balances cumulative regret against movement by allowing projection steps when either the global budget permits or the local suboptimality is small, and triggers jumps only when both conditions are violated.

Similar to Definition \ref{defn:balancejump}, we define \textsc{\textsc{Gap-Frugal}} to play a {\it lazy} action at time $t$ if  selects $x_t= {\hat x}_t$. Otherwise, it is defined to {\it jump} at time  $t$. Similarly, let 
$1 \le t_1 <  \cdots < t_N \le T$ be the time slots at which \textsc{\textsc{Gap-Frugal}} algorithm
{\it jumps}.  

\begin{definition}[{\it Directional} Input]\label{def:directional-interval}
A sequence of nested closed convex sets $S_1 \supseteq S_2 \supseteq \cdots \supseteq S_T$ is 
{\bf directional} with respect to a convex function $f$ if, for any starting index $s$ and any target index $t > s$, the sequence $\{y_\tau\}_{\tau=s}^{t-1}$ starting from the optimizer $y_s \in X_s^*$, where $X_s^*= \arg\min_{x \in S_s} f(x)$ and generated by the successive Euclidean projections
$y_\tau = \Pi_{S_\tau}(y_{\tau-1}) \quad \text{for } \tau = s+1, \dots, t-1,$ satisfies the condition:
\begin{equation}
f(y_\tau) \le v_t \quad \text{for all } \tau \in \{s+1, \dots, t-1\},
\end{equation} where recall that $v_t=\min_{x \in S_t} f(x)$. See Fig. \ref{fig:lb_strong} for an example of a directional input and a use-case in Section \ref{app:proginput}. 
\end{definition}

\begin{rem}[Geometric Intuition]
The Euclidean projection $\Pi_{S_\tau}$ represents the ``laziest" possible movement to regain feasibility. The directional input essentially guarantees that by moving as little as possible to satisfy tightening constraints, the objective value $f$ does not drift higher than the value the algorithm would have achieved had it ``jumped" directly to the future optimizer $x_t^\star$.
\end{rem}

\begin{theorem}\label{thm:directional}For  {\it directional} input, the regret and movement cost of algorithm \textsc{\textsc{Gap-Frugal}} are 
$$\mathrm{Regret}_{\textsc{\textsc{Gap-Frugal}}}(T) \le  O(1), \quad \text{and} \quad M_{\textsc{\textsc{Gap-Frugal}}}(T) = O(\log^2 T).$$
\end{theorem}
Proof of Theorem \ref{thm:directional} is provided in Section \ref{app:thm:directional}.

\section{Linear Combination of Loss Function Cost and Movement Cost}\label{sec:lincost}
So far in this paper, we have considered the simultaneous minimization of regret and movement cost, and tried to characterize the two-dimensional tradeoff region. Next, we consider problem \eqref{eq:lincost} under {\bf dynamic constraints}, i.e. under the constraint that $x_t\in S_t$ with $S_t$'s changing over time. For problem \eqref{eq:lincost}, let the cost of algorithm $\cA$ be
\begin{equation}\label{eq:lincombcost}
\textsf{Lin}_\cA(T) = \sum_t f_t (x_t) + \sum_t ||x_t-x_{t-1}||^\sfp,
\end{equation}
{\bf such that} $x_t\in S_t$ for all $t$. 
The metric for $\cA$ in this case is the competitive ratio that is defined as 
\begin{equation}\label{eq:compratio}
\mu_\cA = \max_{f_t,S_t, x_t\in S_t} \frac{\textsf{Lin}_\cA(T)}{\textsf{Lin}_\opt(T)},
\end{equation}
and the objective is to find online algorithms with small competitive ratios.
We review in Section \ref{sec:ncbclincost} how problem \eqref{eq:lincost} can be solved under dynamic constraints by using the same generic idea of mapping problem \eqref{eq:lincost} to an NCBC problem for $\sfp=1$ via defining {\it work functions}.  The challenge, however, is that finding work functions is prohibitive, and thus, similar to prior work \cite{bansal20152, argue2020dimension} where $S_t$'s do not change over time, we next consider simple algorithms for $d=1$ when $f$ is convex and for general $d$ when $f$ is $\alpha$-sharp. For $\sfp=2$, existing results extend directly to the dynamic constraint setting without modification; see Section~\ref{sec:p2}.

\vspace{-0.1in}
\section{Solving \eqref{eq:lincombcost} with dynamic constraints when $\sfp=1$}\label{sec:lincostresults}
\vspace{-0.1in}
\subsection{A $5$-Competitive Algorithm when $d=1$}\label{sec:5comp}



{\bf Algorithm Definition:}
We need the following definitions.

{\bf  D1:}
    At time $t$, when $S_t$ is revealed, let $z_t = \Pi_{S_t}(x_{t-1})$.
    
{\bf D2:} Without loss of generality we let $f_t(x)\ge 0, \forall x \in \bbR$. At time $t$, define 
{\bf Constrained Feasible Set} $\cF_t$ as 
$ \cF_t = \{ x \in S_t \mid |x - z_t| \le f_t(x) \}.$ Note that $z_t\in \cF_t$ since $f_t(x)\ge 0$.
$\cF_t$ is the set of points in $S_t$ reachable from $z_t$ with movement cost at most the service cost $f_t(x)$. A point $x \in \mathbb{R}$ is said to be \textbf{ $\cF_t$-feasible} (with respect to reference $z_t$) if $x \in \cF_t$.  
Note that $\cF_t$ need not be convex, thus, we need to define the following convex set.

{\bf D3:} Define $\cC_t(z_t)$ as the largest connected component of $\cF_t$ containing $z_t$, i.e. $\cC_t(z_t)= \{x\in \cF_t : [x,z_t]\subseteq \cF_t \ \text{or} \ [z_t,x]\subseteq \cF_t\}$.

{\bf Algorithm $\cA_B$:} $\cA_B$ chooses $x_t$ as the minimizer of $f_t$ over $\cC_t(z_t)$, i.e.
    $x_t = \arg\min_{x \in \cC_t(z_t)} f_t(x).$ If the minimizer is not unique, choose the closest to $x_{t-1}$.

   \begin{theorem}\label{thm:cr5}
    For $d=1$, algorithm $\cA_B$ is $5$-competitive.  
\end{theorem} 
Proof of Theorem \ref{thm:cr5} can be found in Section \ref{app:proofcr5}. Next, we show that the greedy algorithm $\cG$ ($x_t= x_t^\star$ for all $t$) has bounded competitive ratio for all $d$ when $f_t$'s are $\alpha$-sharp (Definition \ref{defn:alphasharp}).

\vspace{-0.1in}
\subsection{$f_t$'s are $\alpha$-sharp}\label{sec:alphasharpcr}

\begin{definition}\label{defn:alphasharp}
Let $S^\star$ denote the set of global minimizers of $f$ in $S$, and let $f^\star_S$ be the minimum value. The function $f$ is said to be $\alpha$-sharp over $S$ if there exists a constant $\alpha > 0$ such that for all $x \in S$:$f(x) - f^\star_S \ge \alpha \cdot \text{dist}(x, S^\star),$ where $\text{dist}(x, S^\star) = \inf_{x^\star \in S^\star} \|x - x^\star\|$ is the shortest distance from $x$ to the optimal set $S^\star$. If the function possesses a unique global minimizer $x^\star$, the definition simplifies to: $\quad f(x) - f(x^\star) \ge \alpha \|x - x^\star\|.$
\end{definition}


\begin{theorem}\label{thm:sharp_comp}
 For any $d$, when $f_t, \ \forall \ t$ are $\alpha$-sharp and possesses a unique global minimizer $x_t^\star$ for each $S_t$, $\cG$ is $\frac{12}{\alpha}$-competitive for \eqref{eq:lincombcost} with dynamic constraints.
\end{theorem}
\vspace{-0.1in}
Proof of Theorem \ref{thm:sharp_comp} can be found in Section \ref{app:sharp_comp}.
\vspace{-0.1in}
%

\vspace{-0.1in}
\section{Conclusion} 
\vspace{-0.1in}In this paper, we introduced CONES, a novel framework for online decision-making in  constraint-adaptive optimization.  We characterized the fundamental trade-off between regret and movement cost, demonstrating that while a greedy approach yields optimal regret, it incurs an unavoidable polynomial movement penalty of $\Theta(\sqrt{T})$. To overcome this limitation, we proposed the \textsc{Frugal} algorithm, which employs selective lazy projections to achieve a minimax optimal $O(\log T)$ movement cost while maintaining non-positive regret for both strongly convex and $\alpha$-sharp loss functions. For general convex functions, we developed the Level Set Projection (LSP) algorithm to establish a controllable $(T^{1-\beta}, T^\beta )$ regret-movement tradeoff, alongside the \textsc{GAP-Frugal} algorithm that achieves $O(1)$ regret and $O(\log^2 T)$ movement cost under directional input structures. However, fully solving the general convex case  remains an important and challenging open problem. Extending this framework to dynamic benchmarks and bounding competitive ratios for higher-dimensional dynamic constraints represent natural and compelling directions for future work.

\bibliographystyle{plain}
\newpage
\bibliography{../references,../oco}
\section{Motivating Applications for CONES}\label{sec:applications}

To motivate CONES, consider a company operating under regulatory oversight, such as an airline or a medical data platform. The company seeks to optimize performance (e.g., profit) while satisfying regulatory constraints. These constraints, however, evolve over time to reflect new policies, safety requirements, or privacy standards, thereby progressively restricting the set of feasible decisions. Mathematically, this mirrors the CONES framework, where the feasible set at time $t$ is a subset of the previous set ($S_t \subseteq S_{t-1}$). 

As the ``feasible envelope'' contracts, actions that were previously optimal may become impermissible, forcing the company to shift its strategy. In complex systems, these shifts are rarely free; they incur movement costs ($M_{\mathcal{A}}$) involving resource reallocation, system downtime, or logistical rescheduling. CONES addresses the fundamental trade-off between achieving low static regret (staying near the global optimum of the final restricted set) and minimizing the cumulative adaptation cost incurred during the transition.

A poignant example of this trade-off is the December 2025 operational crisis faced by IndiGo. Following the DGCA's revised pilot rest regulations, the airline was forced to contract its scheduling parameters. While the regulatory shift was incremental, the ``movement'' required to remain compliant within a tightly coupled system triggered a nationwide meltdown (see, e.g., \cite{bsIndigo}). This underscores the necessity of algorithms that treat movement cost as a primary objective rather than a secondary concern.

We next present two concrete AI/ML applications where CONES
arises naturally and importantly.


\subsection{Margin-Constrained Classification as an Instance of CONES.}

Margin-constrained binary classification fits naturally into the CONES framework by treating the classifier parameters as the decision variable and progressively tightening margin requirements as evolving feasibility constraints. Let $\mathcal{X} \subset \mathbb{R}^d$ be a convex compact set and consider labels $y_i \in \{-1,+1\}$, so that the signed margin $y_i w^\top x_i$ measures classification confidence \cite{cortes1995support,scholkopf2002learning}. The objective is a convex empirical risk
\[
f(w)=\frac{1}{n}\sum_{i=1}^n \ell(y_i, w^\top x_i),
\]
where $\ell$ is a margin-based loss such as hinge or logistic \cite{cortes1995support}.

At each round $t$, we impose a minimum margin requirement $\delta_t \ge 0$ via a convex surrogate constraint:
\[
S_t = \left\{ w \in \mathcal{X} \;\middle|\;
\frac{1}{n}\sum_{i=1}^n \big[\delta_t - y_i w^\top x_i\big]_+ \le \epsilon_t
\right\},
\]
where $[z]_+ = \max(0,z)$. This constraint penalizes violations of the margin requirement and enforces that the average margin shortfall is bounded by $\epsilon_t$. In effect, it penalizes the lower tail of the margin distribution by focusing optimization on low-margin (i.e., hardest) samples \cite{bartlett1998sample}. Increasing $\delta_t$ or decreasing $\epsilon_t$ imposes progressively stricter requirements on the margin distribution, thereby strengthening robustness guarantees in a manner consistent with margin-based generalization theory \cite{bartlett1998sample}.

As the margin requirement tightens over time, the feasible sets shrink monotonically, i.e., $S_t \subseteq S_{t-1}$, forming a nested sequence consistent with CONES. The learner selects $w_t \in S_t$ at each round, trading off empirical risk against movement $\sum_t \|w_t - w_{t-1}\|$, thereby capturing the tension between improving margin guarantees and maintaining model stability under progressively stricter constraints. This perspective is also aligned with robust optimization, where tightening feasibility constraints yields increasingly conservative solutions \cite{ben2009robust}.

\subsection{AI Alignment and Safe Reinforcement Learning from Human Feedback (RLHF).}\label{sec:rlhf}
Modern alignment protocols, such as \citep{bai2022constitutional,dai2023safe}, iteratively refine a language model policy $\pi_\theta$ through repeated rounds of constraint enforcement and policy optimization. Under a fixed (or slowly varying) reward model and reference policy, this process can be naturally viewed through the lens of the CONES framework, where progressively stricter safety and stability requirements induce a sequence of nested feasible regions.

We represent the model by its parameters $\theta \in \mathbb{R}^d$.

\textbf{(i) Evolving Feasible Sets ($S_t$).}
At each round $t$, updated safety requirements—arising from red-teaming, evaluation pipelines, or regulatory constraints—restrict the admissible parameter space. Let $x \in \mathcal{X}$ denote an input (e.g., prompt) drawn from a distribution $\mathcal{D}$, and $y \in \mathcal{Y}$ denote the model output sampled from the policy $\pi_\theta(\cdot \mid x)$. These requirements typically enforce an expected cost threshold $c_t$ (capturing harmful or policy-violating behavior) together with a trust-region budget $\kappa_t$, yielding
\[
S_t = \left\{ \theta \in \Theta :
\mathbb{E}_{x \sim \mathcal{D},\, y \sim \pi_\theta(\cdot|x)}[C(x,y)] \leq c_t, \quad
\mathrm{KL}(\pi_\theta \| \pi_{\mathrm{ref}}) \leq \kappa_t
\right\},
\]
where $C(x,y) \ge 0$ is a cost function that quantifies undesirable or unsafe behavior (e.g., policy violations or harmful outputs).

The cost constraint enforces safety, while the KL constraint limits deviation from a reference policy $\pi_{\mathrm{ref}}$, stabilizing updates via a trust-region mechanism \citep{schulman2015trpo,schulman2017ppo}. Such expected cost constraints serve as convex surrogates for safety filtering or reward shaping commonly used in practice.

While these constraints are not convex in the parameterization $\theta$, trust-region methods locally approximate the KL divergence by a quadratic form (e.g., via the Fisher information matrix), yielding convex feasible regions in a neighborhood of the current iterate. We assume that safety and trust-region thresholds are tightened over time, i.e., $c_t \le c_{t-1}$ and $\kappa_t \le \kappa_{t-1}$, which induces the nesting property
\[
S_t \subseteq S_{t-1}.
\]
This captures regimes in which safety and stability requirements become progressively stricter across alignment rounds.

\textbf{(ii) Objective Function ($f(\theta)$).}
The objective is to maximize reward while regularizing toward the reference policy:
\[
f(\theta) = -\mathbb{E}_{x \sim \mathcal{D},\, y \sim \pi_\theta(\cdot|x)}[R(x,y)]
+ \lambda\, \mathrm{KL}(\pi_\theta \| \pi_{\mathrm{ref}}).
\]
This is the standard KL-regularized RLHF objective \citep{ziegler2019fine,ouyang2022training}, where the reward term promotes preference alignment and the KL term penalizes large deviations from the reference model. The KL divergence thus plays a dual role: it appears as a hard constraint in $S_t$ (trust-region enforcement) and as a soft regularizer guiding optimization.

\textbf{(iii) Movement Cost and Stability ($M_{\mathcal{A}}$).}
A central challenge in alignment is controlling the magnitude of updates across rounds. In practice, this is achieved via trust-region constraints on policy divergence. In parameter space, this corresponds to controlling the displacement $\|\theta_t - \theta_{t-1}\|_2$. Large updates can degrade capabilities—the \emph{alignment tax} \citep{ouyang2022training}—while cumulative drift can lead to catastrophic forgetting \citep{kirkpatrick2017overcoming}.

Related approaches implicitly control such movement: trust-region methods bound KL divergence between successive policies \citep{schulman2015trpo,schulman2017ppo}, while continual learning techniques such as elastic weight consolidation (EWC) penalize deviation from previously important parameters \citep{kirkpatrick2017overcoming}. In our framework, this is captured explicitly through the cumulative movement
\[
M_{\mathcal{A}}(T) = \sum_{t=1}^T \|\theta_t - \theta_{t-1}\|.
\]

\paragraph{Connection to CONES.}
This mapping is conceptual but captures the key structural elements of CONES: actions $x_t$ correspond to parameters $\theta_t$, the loss $f(x_t)$ matches the KL-regularized RLHF objective, and the feasible sets $S_t$ encode safety and trust-region constraints. The nesting $S_t \subseteq S_{t-1}$ reflects progressively tightening requirements, while the movement cost $M_{\mathcal{A}}(T)$ quantifies the total magnitude of policy updates, providing a principled measure of alignment stability. Although modern RLHF is inherently non-convex and may involve evolving reward models, this abstraction isolates the geometric structure of alignment under progressively constrained optimization.

\section{Comparison with Prior Work}\label{sec:classicalpriorwork}
\subsection{Comparison with Sensitivity/Perturbation Analysis of Optimization Problems}

Compared to CONES, the Bonnans-Shapiro framework in monograph \cite{bonnans2000perturbation} (also see \cite{fiacco1983introduction}) provides a complementary but fundamentally different perspective on optimization under evolving constraints. Their work characterizes the sensitivity of optimal solutions $x^\star(\theta)$ to parametric perturbations in problems of the form $$\min_{x \in S(\theta)} f(x,\theta),$$ establishing conditions under which the solution mapping $\theta \mapsto x^\star(\theta)$ exhibits Lipschitz continuity, directional differentiability, or second-order expansions. When constraint sequences admit smooth parametric structure, such as $$S_t = \{x : g(x) \leq u_0 - \delta \cdot h(t)\}$$ for some smooth function $h(t)$---their theory provides explicit Lipschitz constants $L$ such that $$\|x^\star_{t+1} - x^\star_t\| \leq L \cdot |h(t+1) - h(t)|.$$ For instance, when constraints tighten exponentially ($h(t) = 1 - e^{-\lambda t}$), Bonnans-Shapiro analysis yields $\|x^\star_{t+1} - x^\star_t\| = O(\lambda e^{-\lambda t})$, implying that the greedy algorithm $\mathcal{G}$ (which always plays $x_t = x^\star_t$) achieves $M_{\mathcal{G}}(T) = O(1)$ total movement cost when $f$ is strongly convex or $\alpha$-sharp---a dramatic improvement over our worst-case $\Theta(\sqrt{T})$ and $\Theta(\log T)$ bounds respectively (Theorems 2, 3). Similarly, for polynomial saturation $h(t) = 1 - (1-t/T)^\beta$ with $\beta > 1$, the greedy movement becomes $M_{\mathcal{G}}(T) = o(\log T)$.

However, Bonnans-Shapiro theory fundamentally differs from CONES in scope and guarantees. First, it requires \textbf{known parametric structure} $S(\theta)$ with smoothness assumptions and constraint qualifications (e.g., Robinson's constraint satisfaction \cite{bonnans2000perturbation}), whereas CONES provides minimax-optimal algorithms for \textbf{arbitrary adversarial} nested sequences $\{S_t\}$ revealed online without any structural assumptions. Second, Bonnans-Shapiro offers \emph{sensitivity characterizations} (how solutions change) but not \emph{online decision-making algorithms}: it assumes the full constraint family is known a priori, enabling offline sensitivity analysis, while CONES algorithms like \textsc{Frugal} must commit to actions before seeing future constraints. Third, CONES explicitly addresses the \textbf{regret-movement tradeoff} through strategic laziness---\textsc{Frugal} achieves zero regret with $O(\log T)$ movement by selectively jumping only when necessary to maintain feasibility and bounded regret, whereas Bonnans-Shapiro focuses solely on tracking optimal solutions without considering operational switching costs or benchmark competition. The frameworks are complementary: Bonnans-Shapiro identifies \emph{benign problem structures} where exploitation is possible (e.g., smooth regulatory trajectories in practice), while CONES provides \emph{robust worst-case guarantees} when such structure is absent or unreliable, with our lower bounds (Theorems 8, 11) proving that $\Omega(\log T)$ movement is unavoidable for sub-linear regret algorithms facing adversarial constraint sequences even when $f$ is strongly convex or $\alpha$-sharp.


\subsection{Trajectory Tracking}

Simonetto et al.~\cite{simonetto2020time} study the problem of tracking the optimal trajectory of a time-varying optimization problem
\[
x^\star(t) := \arg\min_{x \in X(t)} f(x; t),
\]
where both the objective and the constraint set evolve continuously over time. Their framework focuses on minimizing the \emph{Asymptotic Tracking Error} (ATE),
\[
\lim_{t \to \infty} \sup \|\hat{x}(t) - x^\star(t)\|,
\]
using a time-structured Prediction-Correction (PC) methodology. This approach relies on substantial regularity assumptions: the objective $f$ is required to be $m$-strongly convex and $L$-smooth, and the problem must exhibit temporal smoothness, typically formalized via a bounded mixed partial derivative $\|\nabla_{tx} f(x; t)\| \leq \Delta_0$. Leveraging this structure—often including access to Hessians $\nabla_{xx} f$—PC methods predict the evolution of the optimizer and then apply a correction step upon observing updated problem data.

This setting differs fundamentally from CONES along several dimensions. First, CONES operates in discrete time with arbitrary, potentially adversarial, nested constraint sequences $\{S_t\}$, where no temporal smoothness or predictive structure is assumed. In particular, derivatives such as $\nabla_{tx} f$ may not exist, and the learner must act without anticipating future constraints.

Second, the role of movement is conceptually different. In trajectory tracking, movement is a \emph{means} to maintain accuracy: PC methods update the iterate at every time step to remain within an $O(h^2)$ neighborhood of the optimizer (where $h$ is the sampling period). In contrast, CONES treats movement as a primary \emph{cost},
\[
M(T) = \sum_{t=1}^T \|x_t - x_{t-1}\|,
\]
and seeks to minimize it through strategic laziness, only moving when necessary to control regret.

Finally, the guarantees are complementary. Under smooth temporal evolution, \cite{simonetto2020time} shows that one can achieve vanishing tracking error (and correspondingly small dynamic regret). In contrast, CONES establishes worst-case guarantees: in the absence of temporal structure, an $\Omega(\log T)$ movement cost is unavoidable for any algorithm achieving sub-linear regret (Theorems~8 and~11). 

Taken together, the two frameworks highlight a fundamental dichotomy: trajectory tracking quantifies what is achievable under strong temporal regularity and predictive access, while CONES characterizes the minimal movement required for robust performance under adversarial and non-smooth evolution.

\subsection{Model Predictive Control (MPC)}

Mayne et al.~\cite{mayne2000constrained} establish the foundational framework for Model Predictive Control (MPC), focusing on stability and optimality in constrained discrete-time dynamical systems. In this setting, a controller drives the system state $x$ toward a target (typically the origin) according to dynamics $x_{k+1} = f(x_k, u_k)$, while satisfying state and input constraints ($x \in X$, $u \in U$). At each time step, MPC solves a finite-horizon optimal control problem and applies a receding horizon strategy. By incorporating a terminal cost $V_f(x)$ and a terminal constraint set $X_f$, they guarantee recursive feasibility and asymptotic stability.

This setting differs fundamentally from CONES in both modeling assumptions and objectives. MPC assumes a known and stationary system model, where constraints are fixed and the primary challenge is to plan trajectories that remain feasible while driving the system toward an optimum. In contrast, CONES operates in an online optimization setting where the feasible sets $\{S_t\}$ are time-varying, potentially adversarial, and revealed only sequentially, with no predictive model of their evolution.

A key conceptual difference lies in the role of movement and the nature of feasibility. In MPC, the feasible set $X$ acts as a static ``safe region,'' and movement is dictated by system dynamics and control objectives. The controller proactively plans trajectories using knowledge of the model and constraints. In CONES, however, the feasible set $S_t$ itself evolves over time, effectively becoming a moving target. The learner must adapt to these changes without foresight, balancing performance (regret) against the cost of movement
\[
M(T) = \sum_{t=1}^T \|x_t - x_{t-1}\|.
\]

The resulting guarantees highlight complementary regimes. Under strong structural assumptions, MPC achieves asymptotic stability and vanishing tracking error. In contrast, CONES provides worst-case guarantees: when constraints evolve adversarially and no temporal structure is available, an $\Omega(\log T)$ movement cost is unavoidable for any algorithm achieving sub-linear regret. Thus, while MPC characterizes optimal regulation in structured, model-driven environments, CONES captures the fundamental limits of decision-making in unstructured, dynamically constrained settings.

\section{Computational Complexity and Oracle Models.}
\label{sec:comp}
All the algorithms that we have considered in the paper assume efficient access to two primitives: 
(i) projection onto the current feasible set $\Pi_{S_t}(x)$, and 
(ii) computation of the constrained minimizer 
$x_t^\star = \arg\min_{x \in S_t} f(x)$. 
These correspond to projection and optimization oracles, respectively, and the per-round complexity is dominated by these operations (e.g., \textsc{Frugal} performs both in each round). When projections are computationally expensive—as is the case for general convex bodies —the practical efficiency depends strongly on the structure of $S_t$ (e.g., polyhedral versus general convex sets). From an oracle-complexity perspective, however, our guarantees are largely agnostic: a separation oracle can be used to simulate projections via iterative methods such as cutting-plane or ellipsoid procedures, and optimization oracles can in turn be reduced to separation via standard convex optimization equivalences, albeit with polynomial overhead. Consequently, the regret--movement guarantees remain unchanged across oracle models, though the per-iteration computational cost may increase (e.g., polynomially in dimension and accuracy).

In settings where only first-order information is available, projections and constrained minimizations must typically be implemented approximately, introducing additive errors at each step. Incorporating such errors into the regret budget (Remark \ref{rem:budget}) while preserving feasibility and non-positive regret guarantees is technically delicate, as it requires a stability analysis of both the budget evolution and the phase structure underlying \textsc{Frugal}-type algorithms. In particular, even small inaccuracies can accumulate across phases and potentially disrupt the geometric growth arguments used to establish logarithmic movement bounds. Developing a robust analysis for such inexact oracle models—especially under adversarially evolving feasible sets—therefore presents non-trivial challenges beyond the current framework.

There is, however, strong precedent in optimization and online learning for extending projection-based methods to weaker oracle models. In convex optimization, projection-based methods such as projected gradient descent have been generalized to projection-free approaches like the Frank--Wolfe algorithm~\cite{frank1956algorithm,jaggi2013revisiting}, which rely only on linear minimization oracles. Similarly, in online learning, classical Online Convex Optimization with full-information feedback has been extended to the bandit setting, where only zeroth-order feedback is available~\cite{flaxman2005online,agarwal2010optimal}. These developments suggest that analogous generalizations of CONES to first-order, projection-free, or bandit feedback models should be possible, but will likely require new algorithmic ideas to control the interaction between approximation error, regret, and movement cost.

Since the present work is primarily intended to establish the fundamental limits of the regret--movement tradeoff in CONES, we defer a systematic treatment of such oracle-efficient and bandit variants to future work.
\section{Dependence on the Choice of Norm}\label{sec:norm}
All guarantees in CONES are stated with respect to the Euclidean norm $\|\cdot\|_2$, both for measuring movement cost and for defining projections. Several key steps of the analysis rely on geometric properties specific to this choice, including the use of Euclidean projections $\Pi_S(\cdot)$, smoothness and strong convexity defined relative to $\|\cdot\|_2$, and bounds on projection curves that scale as $O(d^{d/2}\,\mathrm{diam}(S_1))$. 

Extending the framework to a general norm $\|\cdot\|$ is conceptually natural but technically non-trivial. The main ingredients must be replaced by their norm-dependent counterparts: projections become norm projections or, more generally, Bregman projections (which may be non-unique unless the norm is strictly convex).While the core structural components of the analysis—such as the regret budget accounting (Remark \ref{rem:budget}), phase decomposition, and geometric growth of jump times—are not inherently tied to the Euclidean setting, their quantitative behavior depends critically on the underlying geometry.

In particular, the projection-curve bound (Lemma~\ref{lem:projdistanceint}), which plays a central role in controlling movement within phases, is intrinsically Euclidean and would need to be replaced by an appropriate bound in the given normed space. Such bounds are typically governed by the modulus of convexity/smoothness of the norm and the geometry of its unit ball, and may introduce different (and potentially worse) dependence on the dimension. 

Overall, we expect the qualitative guarantees—such as $O(\log T)$ movement for strongly convex or $\alpha$-sharp objectives—to extend to general normed spaces. However, doing so requires reworking the geometric components of the analysis and may lead to different constants and dimension-dependent factors.
\section{Proof of Theorem \ref{thm:lbGreedy}}\label{app-lbstrconvxGreedy}
\begin{proof}
Consider the following input construction 
i) {\bf $d=2$}.
    ii) \textbf{Admissible set} $\mathcal{X}\in \mathbb{R}^2$ is a rectangle that does not contain the origin, and iii) \textbf{Cost function} $f_t(.) = f(x,y) = x^2+y^2$ for all $t$.
For this construction, the starting point is assumed to be $x_0^\star = \arg\min_{x\in\cX}f(x)$. The derived results still hold, as for any initialization, the actual movement cost can be off by at most $D$ from the computed value due to the fact that $\mathcal{X}$ is bounded. 

The first few steps of the nesting scheme ($S_1, S_2, \dots, S_T$) for the lower bound are depicted in Fig. \ref{fig:lb_strong}. The contours of the chosen function $f(x,y) = x^2+y^2$ are concentric circles. This implies that the local minimizer $x_t^\star$ of $f$ over a set $S_t$ is the point in the set that is closest to the origin. The key idea in choosing sets $S_t$ is to keep the local minimizers oscillating between the lines $L_1$ and $L_2$ shown in Fig. \ref{fig:lb_strong}, separated by a horizontol  distance of $k$. The new set $S_t$ is formed by removing the portion of the current set $S_{t-1}$ above the tangent to the desired local minimizer with respect to the origin.
\begin{center}
\begin{figure}
\begin{tikzpicture}[scale=1.2]
    \foreach \r in {0.5, 1, 1.5, 2, 2.5} {
        \draw[blue!40!black, dashed] (0,0) circle (\r);
    }
    \filldraw[fill=orange!20, fill opacity=0.5, draw=orange!80!black, thick]
        (-1, -0.5) rectangle (1, -2.5);
    \fill (0, -0.5) circle[radius=1pt];
    \node[above right] at (0, -0.5) {$x_0^\star$};
    \fill (0,0) circle[radius=1pt];
    \node[above] at (0,-0.05) {O};
    \filldraw[fill=blue!20, fill opacity=0.5, draw=blue!80!black, thick] (-0.5, -0.5) -- (-1, -0.5) -- (-1, -2.5) -- (1, -2.5)-- (1,-2)--cycle;
    \fill (-0.5, -0.5) circle[radius=1pt];
    \node[above left] at (-0.5, -0.5) {$x_1^\star$};
    \filldraw[fill=green!20, fill opacity=0.5, draw=green!80!black, thick] (0, -1) -- (-1, -1)--(-1, -2.5) -- (1, -2.5)--(1,-2)-- cycle;
    \fill (0, -1) circle[radius=1pt];
    \node[above right] at (0, -1) {$x^\star$};
    \filldraw[fill=red!20, fill opacity=0.7, draw=red!80!black, thick] (-0.5, -1) -- (-1, -1)--(-1, -2.5) --  (1,-2.5) -- (1, -2)--(0.5, -1.5)-- cycle;
    \fill (-0.5, -1) circle[radius=1pt];
    \node[above left] at (-0.5, -1) {$x_3^\star$};
    \filldraw[fill=cyan!20, fill opacity=0.7, draw=cyan!80!black, thick] (0, -1.25) -- (-1, -1.25)--(-1, -2.5) --  (1,-2.5) -- (1,-2)-- (0.5, -1.5)-- cycle;
    \fill (0, -1.25) circle[radius=1pt];
    \node[below left] at (0, -1.25) {$x_4^\star$};
    \draw[dotted, thick] (0,0) -- (-0.5, -0.5);
    \draw[dotted, thick] (0,0) -- (-0.5, -1);
    \draw (-0.4, -0.4)--(-0.3,-0.5);
    \draw (-0.4, -0.6)--(-0.3,-0.5);
    \draw (-0.45, -0.9)--(-0.35, -0.95);
    \draw (-0.4, -1.05)--(-0.35, -0.95);
    \draw[black!100, thick] (0, -0.5)--(-0.5,-0.5) -- (0, -1) --(-0.5, -1)-- (0,-1.25);
    \draw[black!100, line width=0.6pt] (-1.1, 0.32) rectangle (0.6, -1.62);
    \draw[black!100, line width=0.6pt] (0.6, 0.32) -- (6.7, -0.1);
    \draw[black!100, line width=0.6pt] (-1.1, -1.62) -- (3.3, -3.99);
\begin{scope}[shift={(3.5,2)}, scale=0.8] 
    \draw[fill=orange!20] (0,0) rectangle (0.2,0.2);
    \node[right=1pt] at (0.4,0.1) {Admissible Set $\mathcal{X}$};
    \draw[fill=blue!20] (0,-0.4) rectangle (0.2,-0.2);
    \node[right=1pt] at (0.4,-0.3) {Set $S_1$};
    \draw[fill=green!20] (0,-0.8) rectangle (0.2,-0.6);
    \node[right=1pt] at (0.4,-0.7) {Set $S$};
    \draw[fill=red!20] (0,-1.2) rectangle (0.2,-1);
    \node[right=1pt] at (0.4,-1.1) {Set $S_3$};
    \draw[fill=cyan!20] (0,-1.6) rectangle (0.2,-1.4);
    \node[right=1pt] at (0.4,-1.5) {Set $S_4$};
  \end{scope}
  \begin{scope}[shift={(5.5,-0.75)}, scale=2] 
    \fill (0,0) circle[radius=1pt];
    \node[above] at (0,0.05) {O};
    \fill (0, -0.5) circle[radius=1pt];
    \node[above right] at (0, -0.5) {$x_0^\star$};
    \draw[red, dashed] (0, 0) -- (0, -1.5);
    \draw[red, dashed] (-0.5, 0) -- (-0.5, -1.5);
    \fill (-0.5, -0.5) circle[radius=1pt];
    \node[above left] at (-0.5, -0.5) {$x_1^\star$};
    \fill (0, -1) circle[radius=1pt];
    \node[above right] at (0, -1) {$x_2^\star$};
    \fill (-0.5, -1) circle[radius=1pt];
    \node[above left] at (-0.5, -1) {$x_3^\star$};
    \fill (0, -1.25) circle[radius=1pt];
    \node[below left] at (0, -1.25) {$x_4^\star$};
    \draw[dotted, thick] (0,0) -- (-0.5, -0.5);
    \draw[dotted, thick] (0,0) -- (-0.5, -1);
    \draw[->, thick] (0,-0.4) arc[start angle=-90, end angle=-160, radius=0.4cm];
    \node[below] at (-0.35, 0.1) {$\theta_1$};
    \draw[->, thick] (-0.1, -0.2) arc[start angle=-120, end angle=-30, radius=0.2cm];
    \node[right] at (0.1, 0) {$\theta_2$};
    \draw[black!100, thick] (0, -0.5)--(-0.5,-0.5) -- (0, -1) --(-0.5, -1)-- (0,-1.25);
    \draw[black!100, line width=0.6pt] (-1.1, 0.32) rectangle (0.6, -1.62);
    \draw (-0.4, -0.4)--(-0.3,-0.5);
    \draw (-0.4, -0.6)--(-0.3,-0.5);
    \draw (-0.45, -0.9)--(-0.35, -0.95);
    \draw (-0.4, -1.05)--(-0.35, -0.95);
    \draw[<->, thick] (0.25, -0.5) -- (0.25, -1) node[midway, right] {$\Delta r_1$};
    \draw[<->, thick] (0.25, -1) -- (0.25, -1.25) node[midway, right] {$\Delta r_2$};
    \node[right, red] at (0, -1.35) {$L_1$};
  \node[left, red] at (-0.5, -1.35) {$L_2$};
  \end{scope}
  \draw[red, dashed] (0, 0) -- (0, -3);
  \draw[red, dashed] (-0.5, 0) -- (-0.5, -3);
  \node[right, red] at (0, -2.75) {$L_1$};
  \node[left, red] at (-0.5, -2.75) {$L_2$};
  \draw[<->, red] (-0.5, -2.8) -- (0, -2.8) node[midway, above] {$k$};
  \draw[<->, thick, red] (-1.5, -0.5) -- (-1.5, -2.5) node[midway, left] {$\frac{D}{\sqrt2}$};
  \draw[<->, thick, red] (-1, -3) -- (1, -3) node[midway, below] {$\frac{D}{\sqrt{2}}$};
\end{tikzpicture}
\caption{Nesting Scheme}
\label{fig:lb_strong}
\end{figure}
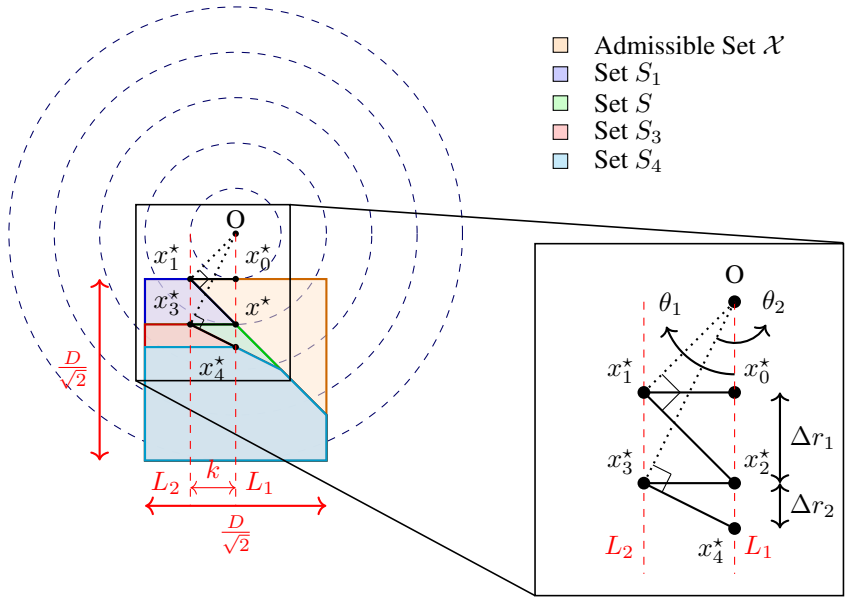
\end{center}
We next describe this construction formally.
Let the feasible set $\cX$ be 
\[
\cX := \Bigl\{(p,q)\in\mathbb{R}^2 : |p|\le \tfrac{D}{2\sqrt{2}},\;
    -a-\tfrac{D}{\sqrt{2}} \le q \le -a \Bigr\},
\]
for some $a>0$. We define two vertical lines 
$$L_1: \{(p,q): p=0\}\qquad L_2: \{(p,q): p=-k\},$$
where the consecutive minimizers of $f$ over $S_t$'s will lie.

Let the initial minimizer (starting point) be $x_0^\star = (0,-a) \in L_1$.
We want $S_1$ such that the next minimizer $x_1^\star := (-k,-a) \in L_2$. To force this, we define 
$
S_1 := \cX \;\cap\;
\bigl\{\,x\in\mathbb{R}^2 : x_1^\star \cdot x \ge \|x_1^\star\|^2 \bigr\}$ since the function $f$ has circular level sets, the supporting tangent at \(x_1^\star\) is
$T_{x_1^\star}: x_1^\star \cdot x = \|x_1^\star\|^2$.

Going forward, let at timestep $t$, the minimizer of set $S_t$, $x_t^\star$ be such that $x_t^\star\in L_1$. Then, the scheme to pick the next minimizer $x_{t+1}^\star\in L_2$ and the next constraint set $S_{t+1}$ is as follows. The tangent to the function contours at $x_t^\star$ has the equation:
$T_{x_t^\star}: x_t^\star\cdot x = \Vert x_t^\star\Vert^2.$
The desired minimum for $S_{t+1}$ is then,
$x_{t+1}^\star = T_{x_t^\star}\cap L_2.$
The tangent at $x_{t+1}^\star$ is similarly, $T_{x_{t+1}^\star}: x_{t+1}^\star\cdot x = \Vert x_{t+1}^\star\Vert^2$.
The new set $S_{t+1}$ is then defined as:
$S_{t+1} := S_t\cap\{x:x_{t+1}^\star\cdot x \ge \Vert x_{t+1}^\star\Vert^2\}.$
Similarly, if $x_t^\star\in L_2$, the desired minimum $x_{t+1}^\star\in L_1$ is defined as $x_{t+1}^\star = T_{x_t^\star}\cap L_1,$ and the procedure to define the new set is identical to that given above.

We define $r_i$ as the distance between the origin and $x_{2i}^\star$, the local minimizer at time step $t=2i$, (with $x_0^\star$ being the minimizer of $f$ over $\cX$), which belongs to the vertical line $L_1$, i.e.,
$r_i = \Vert x_{2i}^\star\Vert,$
and $\Delta r_i = r_i - r_{i-1}.$ Recall that we are assuming that $||\nabla f(x)|| \le G$ for all $x\in \cX$, thus, $r_0$ is a constant.
 

\begin{lemma}\label{lemma:delta_r}
    $\Delta r_i = \frac{k^2}{r_{i-1}}.$
\end{lemma}
\begin{proof}
We  define $\theta_i$ as the angle made by $x^\star_{2i-2}$ and $x^\star_{2i-1}$ at the origin as shown in Fig. \ref{fig:lb_strong},
$$\theta_i = \cos^{-1}{\frac{\Vert x^\star_{2i-2}\Vert}{\Vert x^\star_{2i-1}\Vert}} = \cos ^{-1}{\frac{\Vert x^\star_{2i-1}\Vert}{\Vert x^\star_{2i}\Vert}}.$$
  Thus, from Fig. \ref{fig:lb_strong}, we get 
$$\cos{\theta_i} = \frac{r_{i-1}}{\sqrt{r_{i-1}^2+k^2}}=\frac{\sqrt{r_{i-1}^2+k^2}}{r_i},$$
$$\implies r_i  = r_{i-1} + \frac{k^2}{r_{i-1}}, \text{and} \ \implies \Delta r_i = r_i- r_{i-1} = \frac{k^2}{r_{i-1}}.$$
\end{proof}

The movement cost incurred by $\cG$ at time $t$ (odd) is $k$ while it is $\sqrt{k^2+\Delta r_{t/2}^2}$ otherwise (the hypotenuse of the triangle formed by $x_t^\star$, $x_{t-1}^\star$, and $x_{t-2}^\star$). Thus, assuming $T$ is even,  the total movement cost incurred by $\cG$ is
  \begin{equation}\label{eq:lbMk}
        M_{\mathcal{G}}(T) = \sum_{i=1}^{T/2} \left(\sqrt{k^2 + \Delta r_i^2} + k \right)\ge kT.
\end{equation}

Note that in this setup, the parameter $k$ (which measures the separation between lines $L_1$ and $L_2$) also controls the $\Delta r_i$ incurred between even timesteps. For any fixed $T$, let $k$ be such that the actions chosen by $\cG$, $x_t=x_t^\star \in \cX$ for $t=1,\dots, T$, but $x_{T+2} \notin \cX$. Formally, let $k$ be such that  
\begin{equation}\label{eq:choicek1}
        \sum_{i=1}^{T/2} \Delta r_i \le \frac{D}{\sqrt{2}}, \text{while} \  
        \sum_{i=1}^{T/2 + 1} \Delta r_i > \frac{D}{\sqrt{2}},
\end{equation}
since the height of $\cX$ is $\frac{D}{\sqrt{2}}$.
    Substituting $\Delta r_i = \frac{k^2}{r_{i-1}}$ gives
 \begin{equation}\label{eq:boundk}
 k^2 \sum_{i=1}^{T/2 + 1} \frac{1}{r_{i-1}} > \frac{D}{\sqrt{2}}.
\end{equation}

    Since $r_{i-1}$ is nondecreasing in $i$, we have
       $ \frac{1}{r_{i-1}} \le \frac{1}{r_0}, 
        \ \forall i \le \frac{T}{2} + 1.$
    Hence, $
        \sum_{i=1}^{T/2 + 1} \frac{1}{r_{i-1}} 
        \le \frac{T/2 + 1}{\,r_0\,}.$
    Plugging this into \eqref{eq:boundk}, we get the $k$ that satisfies \eqref{eq:choicek1} as
\begin{equation}
        k^2 \cdot \frac{T/2 + 1}{r_0} > \frac{D}{\sqrt{2}}
        \quad \implies \quad
        k > \sqrt{ \frac{ \frac{D}{\sqrt{2}} r_0}{ \frac{T}{2} + 1 } }.
\label{eq:k_value_lb}
\end{equation}

    Finally, plugging this into \eqref{eq:lbMk}, we have $
        M_{\mathcal{G}}(T)\ge kT
        > T\sqrt{ \frac{ \frac{D}{\sqrt{2}} r_0}{ \frac{T}{2} + 1 } }
        = \Omega(\sqrt{T})$.

\end{proof}

\section{Proof of Theorem \ref{lem:ubG}}\label{app:ubstrconvxG}
\begin{proof}
Let $f$ be $\mu$-strongly convex.
We make use of the following properties of $f$:
    \begin{itemize}
        \item From the convexity of $f$ (Proposition 1.3 in \cite{bubeck2015convexoptimizationalgorithmscomplexity}), for a set $S_t$ with minimizer $x_t^\star$,
        $$\langle \nabla f(x_t^\star), x_t^\star-y\rangle \le 0 \text{\ \ \ for all $y\in S_t$}.$$
        \item From $\mu$-strong convexity of $f$,
        $
        f(x) - f(y) \le \langle\nabla f(x), x-y\rangle -\frac{\mu}{2}\Vert x-y\Vert^2.$
    \end{itemize}
    Combining the two results, we have
    $
    f(y) - f(x_t^\star) \ge \frac{m}{2}\Vert y-x_t^\star\Vert^2\text{\ \ \ for all $y\in S_t$}.
    $
    Noting that $x^\star_{t+1}\in S_t$ due to the nestedness of the constraint sets, we have,
    $
    f(x_{t+1}^\star) - f(x_t^\star) \ge \frac{m}{2}\Vert x_{t+1}^\star - x_t^\star\Vert^2.$
    Summing from $t=1$ to $T-1$,
    \begin{align*}
    f(x_T^\star) - f(x_1^\star) &\ge \frac{\mu}{2}\sum_{t=1}^{T-1}\Vert x_{t+1}^\star-x_{t}^\star\Vert^2 \overset{(a)}{\ge} \frac{\mu}{2{T}}\left(\sum_{t=1}^{T-1}\Vert x_{t+1}^\star - x_t^\star\Vert\right)^2 = \frac{\mu}{2T}(M_\cG({[2:T]}))^2
    \end{align*}
    where inequality (a) follows from the Cauchy-Schwarz inequality. Since $\mathcal{X}$ has a diameter $D$, $\Vert x_T^\star - x_1^\star\Vert\le D$. From the Lipschitzness of $f$, we therefore have $f(x_T^\star) - f(x_1^\star)\le GD$. This implies,
    $$M_\cG([2:T]) \le \sqrt{\frac{2GDT}{\mu}}.$$
    Moreover, appealing to diameter of $\cX$ being $D$, the movement cost incurred at the first time step cannot be greater than $D$. This implies, $M_\cG(T) \le \sqrt{\frac{2GDT}{\mu}}+D$, which concludes the proof.
\end{proof}
\section{Proof of Lemma \ref{lem:regretBalance}}\label{app:lem:regretBalance}
\begin{proof}
Recall that $v_t = \min_{x \in S_t}f(x)$. Then $C_\opt(t) \ge t \cdot v_t$ for all $t\le T$. Thus, to  prove the Lemma we show that the cost of \textsc{Frugal} at time $t$, $F_t = \sum_{\tau=1}^t f(x_\tau) \le t\,v_t$ by induction on $t$.

\emph{Base case $t=1$.}
Since $F_0 = 0$ and $f(\hat{x}_1) \ge v_1$, the {\it lazy} condition
$F_0 + f(\hat{x}_1) \le v_1$ holds iff $f(\hat{x}_1) = v_1$, i.e.,
$\hat{x}_1$ is a minimizer of $f$ over $S_1$.
If it holds, $F_1 = f(\hat{x}_1) = v_1 = 1\cdot v_1$.
If it fails, the algorithm jumps: $F_1 = v_1 = 1\cdot v_1$.
Thus, in both cases $F_1 = v_1 \le 1\cdot v_1$. 

\emph{Inductive step.}
Assume $F_{t-1} \le (t-1)v_{t-1}$.
Since $S_t \subseteq S_{t-1}$, $x_t^\star \in S_{t-1}$, so $v_t \ge v_{t-1}$ and:
$$F_{t-1} \le (t-1)v_{t-1} \le (t-1)v_t.$$

\textsc{Frugal} either uses the {\it lazy} action or {\it jumps} to local minimizer:
 
i)
{\it lazy} action, i.e. $F_{t-1} + f(\hat{x}_t) \le t\,v_t$, then $F_t = F_{t-1}+f(\hat{x}_t) \le t\,v_t$, and 

ii)
\emph{jumps} which means that $x_t = x_{\text{near}, t}^\star$ with $f(x_{\text{near}, t}^\star) = v_t$ and:
$$F_t = F_{t-1} + v_t \le (t-1)v_t + v_t = t\,v_t.$$

In both cases $F_t \le t\,v_t$, so $\mathrm{Regret}_{\textsc{Frugal}}(t)\le 0$. 
\end{proof}
\section{Proof of Theorem \ref{thm:rsb_sc}}\label{appthmrsbsc}
Since $f$ is strongly convex, its minimizer over a convex set is unique. Thus, $x_t^\star= \arg \min_{x\in S_t} f(x)$  and hence 
$x_{\text{near},t}^\star = x_t^\star$ for all $t$. Thus, for this proof, we replace $x_{\text{near},t}^\star$ with $x_t^\star$ for all $t$. Recall that $v_t=f(x_t^\star)$.
Towards proving Theorem \ref{thm:rsb_sc}, we prove the following Lemma that controls the movement cost of  {\it heads} of each phase.
\begin{lemma}\label{lem:mcBalanceStrongHead}
Let $f: \mathcal{X} \to \mathbb{R}$ be $G$-Lipschitz, $L$-smooth, and $\mu$-strongly
convex over $\mathcal{X}$ with diameter
$D = \max_{x,y \in S_1}\|x-y\|$. Then
    $$M_{{\it head}}(T) \triangleq \sum_{k=1}^{N-1}\delta_k
    \;\le\; \frac{2G+(L+\mu)D}{\mu}\log T,$$
    where $\delta_k = \|x_{t_{k+1}}^\star - x_{t_k}^\star\|$ is the distance between
     {\it heads} of phases $k$ and $k+1$, and $N$ is the total number of phases encountered by \textsc{Frugal} by time $T$.
\end{lemma}

\begin{proof}

Between {\it jump} times  $t_k$ and $t_{k+1}$, define the {\it lazy} sequence:
$$\ell_{t_k} := x_{t_k}^\star, \qquad
\ell_\tau := \Pi_{S_\tau}(x_{\tau-1}) \quad\text{for } \ \tau \in \{t_k+1,\dots,t_{k+1}\}.$$
For $\tau \in \{t_k+1,\dots,t_{k+1}-1\}$ the {\it lazy} condition holds so
$x_\tau = \ell_\tau$, giving:
$$F_{t_{k+1}-1} = F_{t_k} + \sum_{\tau=t_k+1}^{t_{k+1}-1}f(\ell_\tau).$$
The {\it jump} at $t_{k+1}$ is triggered because:
$$F_{t_{k+1}-1} + f(\ell_{t_{k+1}}) > t_{k+1}\,v_{t_{k+1}},$$
which after substituting $F_{t_{k+1}-1}$ gives:
\begin{equation}\label{eq:rawtrig}
F_{t_k} + \sum_{\tau=t_k+1}^{t_{k+1}} f(\ell_\tau) > t_{k+1}\,v_{t_{k+1}}.
\end{equation}
From Lemma \ref{lem:regretBalance}, $F_{t_k} \le t_k\,v_{t_k}$ for all $t_k$, so
$-F_{t_k} \ge -t_k\,v_{t_k}$.
Applying this to \eqref{eq:rawtrig} and subtracting
$(t_{k+1}-t_k)v_{t_{k+1}}$ from both sides:
\begin{equation}\label{eq:trig}
\sum_{\tau=t_k+1}^{t_{k+1}} \bigl(f(\ell_\tau) - v_{t_{k+1}}\bigr)
> t_k\bigl(v_{t_{k+1}} - v_{t_k}\bigr).
\end{equation}

Next, we show $\|\ell_\tau - x_{t_{k+1}}^\star\| \le \delta_k$ for all
$\tau \in \{t_k,\dots,t_{k+1}\}$ by induction on $\tau$.

It is clearly true for the base case  $\tau = t_k$ since 
$\|\ell_{t_k} - x_{t_{k+1}}^\star\| = \|x_{t_k}^\star - x_{t_{k+1}}^\star\| = \delta_k$. 

Next, we use the inductive hypothesis.
Since $S_{t_{k+1}} \subseteq S_\tau$ for $\tau \in \{t_k,\dots,t_{k+1}\}$, we have $x_{t_{k+1}}^\star \in S_\tau$ for $\tau \in \{t_k,\dots,t_{k+1}\}$.
By the non-expansiveness of Euclidean projection onto a convex set:
$$\|\ell_\tau - x_{t_{k+1}}^\star\|
= \|\Pi_{S_\tau}(x_{\tau-1}) - x_{t_{k+1}}^\star\|
\le \|x_{\tau-1} - x_{t_{k+1}}^\star\| \le \delta_k,$$ for 
$\tau \in \{t_k,\dots,t_{k+1}\}$.

Hence:
\begin{equation}\label{eq:contract}
\|\ell_\tau - x_{t_{k+1}}^\star\| \le \delta_k
\quad \text{for all } \tau \in \{t_k,\dots,t_{k+1}\}.
\end{equation}

Next, we upper bound the LHS of \eqref{eq:trig}.
By $L$-smoothness of $f$ at $y = x_{t_{k+1}}^\star$:
$$f(\ell_\tau) - v_{t_{k+1}}
\le \langle\nabla f(x_{t_{k+1}}^\star),\, \ell_\tau-x_{t_{k+1}}^\star\rangle
   + \frac{L}{2}\|\ell_\tau-x_{t_{k+1}}^\star\|^2.$$
By $G$-Lipschitzness, $\|\nabla f(x_{t_{k+1}}^\star)\| \le G$.
Applying Cauchy--Schwarz and \eqref{eq:contract}:
\begin{equation}\label{eq:dummy1}
f(\ell_\tau) - v_{t_{k+1}}
\le G\|\ell_\tau-x_{t_{k+1}}^\star\| + \frac{L}{2}\|\ell_\tau-x_{t_{k+1}}^\star\|^2
\le G\delta_k + \frac{L}{2}\delta_k^2.
\end{equation}
Let $\Delta_k = t_{k+1}- t_k$. Then, summing \eqref{eq:dummy1} over $\tau = t_k+1,\dots,t_{k+1}$:
\begin{equation}\label{eq:lhs}
\sum_{\tau=t_k+1}^{t_{k+1}}\bigl(f(\ell_\tau)-v_{t_{k+1}}\bigr)
\le \Delta_k\!\left(G\delta_k + \frac{L}{2}\delta_k^2\right).
\end{equation}

Now, we lower bound on the RHS of \eqref{eq:trig}.
Since $S_{t_{k+1}} \subseteq S_{t_k}$, $x_{t_{k+1}}^\star$ is feasible for $S_{t_k}$.
By first-order optimality of $x_{t_k}^\star$ over $S_{t_k}$:
$$\langle\nabla f(x_{t_k}^\star),\, x_{t_{k+1}}^\star - x_{t_k}^\star\rangle \ge 0.$$
By $\mu$-strong convexity at $y = x_{t_k}^\star$, $x = x_{t_{k+1}}^\star$:
\begin{equation}\label{eq:rhs}
v_{t_{k+1}} - v_{t_k}
= f(x_{t_{k+1}}^\star) - f(x_{t_k}^\star)
\ge \langle\nabla f(x_{t_k}^\star),\, x_{t_{k+1}}^\star-x_{t_k}^\star\rangle
   + \frac{\mu}{2}\delta_k^2
\ge \frac{\mu}{2}\delta_k^2.
\end{equation}

Combining \eqref{eq:trig}, \eqref{eq:lhs}, and \eqref{eq:rhs},
$$\Delta_k\!\left(G\delta_k + \frac{L}{2}\delta_k^2\right)
\;\ge\; \sum_{\tau=t_k+1}^{t_{k+1}}\bigl(f(\ell_\tau)-v_{t_{k+1}}\bigr)
\;>\; t_k(v_{t_{k+1}}-v_{t_k})
\;\ge\; \frac{\mu}{2}\,t_k\,\delta_k^2.$$
Since $\delta_k > 0$ (otherwise a {\it jump} will not happen), divide both sides by
$t_k(G\delta_k + \frac{L}{2}\delta_k^2) = \frac{t_k\delta_k}{2}(2G+L\delta_k) > 0$:
\begin{equation}\label{eq:ratio}
\frac{\Delta_k}{t_k}
> \frac{\frac{\mu}{2}\delta_k^2}{G\delta_k+\frac{L}{2}\delta_k^2}
= \frac{\mu\delta_k}{2G+L\delta_k}.
\end{equation}
Hence:
\begin{equation}\label{eq:growth}
t_{k+1} = t_k + \Delta_k > t_k\left(1 + \frac{\mu\delta_k}{2G+L\delta_k}\right).
\end{equation}

Applying \eqref{eq:growth} iteratively from $k=1$ to $k=N-1$, where $N$ is the total number of {\it jumps} made by \textsc{Frugal} over the time horizon $[1:T]$,:
$$t_N > t_1\prod_{k=1}^{N-1}\left(1+\frac{\mu\delta_k}{2G+L\delta_k}\right)
\ge \prod_{k=1}^{N-1}\left(1+\frac{\mu\delta_k}{2G+L\delta_k}\right),$$
where the last inequality uses $t_1 \ge 1$.
Since $t_N \le T$, taking logarithms:
\begin{equation}\label{eq:logprod}
\log T \ge \log t_N
> \sum_{k=1}^{N-1}\log\!\left(1+\frac{\mu\delta_k}{2G+L\delta_k}\right).
\end{equation}

Note that for any $x \ge 0$, $\log(1+x) \ge \frac{x}{1+x}$
(since $h(x) = \log(1+x) - \frac{x}{1+x}$ satisfies $h(0)=0$
and $h'(x) = \frac{x}{(1+x)^2} \ge 0$).
Setting $x = \frac{\mu\delta_k}{2G+L\delta_k}$:
$$\log\!\left(1+\frac{\mu\delta_k}{2G+L\delta_k}\right)
\ge \frac{\frac{\mu\delta_k}{2G+L\delta_k}}{1+\frac{\mu\delta_k}{2G+L\delta_k}}
= \frac{\mu\delta_k}{2G+(L+\mu)\delta_k}.$$
Since $\delta_k \le D$, the denominator $2G+(L+\mu)\delta_k \le 2G+(L+\mu)D$, so:
$$\frac{\mu\delta_k}{2G+(L+\mu)\delta_k}
\ge \frac{\mu\delta_k}{2G+(L+\mu)D}
=: C\,\delta_k,$$
where $C = \frac{\mu}{2G+(L+\mu)D} > 0$.
Substituting into \eqref{eq:logprod}:
$$\log T > C\sum_{k=1}^{N-1}\delta_k.$$
Rearranging:
\begin{equation}\label{eq:movbound}
M_{{\it head}}(T) = \sum_{k=1}^{N-1}\delta_k 
< \frac{\log T}{C}
= \frac{2G+(L+\mu)D}{\mu}\log T.
\end{equation}
\end{proof}

In light of Lemma \ref{lem:mcBalanceStrongHead}, to prove Theorem \ref{thm:rsb_sc}, we need to upper bound the movement cost within each phase that we do so as follows.
\begin{lemma}\label{lem:mcBalanceStrongPhase}
Let the movement cost within phase $k$ of the \textsc{Frugal} algorithm be 
$$M_{k} = \sum_{\tau=t_k+1}^{t_{k+1}}||x_{\tau}-x_{\tau-1}||.$$
Then
    $$M_{k}  \le O(d^{d/2} \delta_k).$$
\end{lemma}

To prove Lemma \ref{lem:mcBalanceStrongPhase} using Lemma \ref{lem:projdistanceint}, we need the following result.

\begin{lemma}\label{lem:phase-confinement}
Phase $k$ of \textsc{Frugal} spans time slots $\tau \in \{t_k,\dots,t_{k+1}-1\}$. The actions $x_\tau$ for $\tau \in \{t_k,\dots,t_{k+1}-1\}$,  played during phase $k$ lie within a ball of radius
$2\delta_k$ centred at $x_{t_k}^\star$, i.e., 
$$\|x_\tau - x_{t_k}^\star\| \le 2\delta_k
\qquad\forall\,\tau \in \{t_k, t_k+1, \dots, t_{k+1}-1\}.$$
\end{lemma}

\begin{proof}
Phase $k$ runs from $\tau = t_k$ to $\tau = t_{k+1}-1$.
The first iterate is $x_{t_k} = x_{t_k}^\star$ (the forced {\it jump} at the start of the
phase), so $\|x_{t_k} - x_{t_k}^\star\| = 0 \le 2\delta_k$. 

For $\tau \in \{t_k+1,\dots,t_{k+1}-1\}$, \textsc{Frugal} plays {\it lazily}:
$x_\tau = \ell_\tau = \Pi_{S_\tau}(x_{\tau-1})$.
From \eqref{eq:contract}, we have 
$$\|\ell_\tau - x_{t_{k+1}}^\star\| \le \delta_k
\qquad\forall\,\tau\in\{t_k,\dots,t_{k+1}-1\}.$$
Applying the triangle inequality:
$$\|x_\tau - x_{t_k}^\star\|
= \|\ell_\tau - x_{t_k}^\star\|
\le \|\ell_\tau - x_{t_{k+1}}^\star\| + \|x_{t_{k+1}}^\star - x_{t_k}^\star\|
\le \delta_k + \delta_k
= 2\delta_k,$$
$\forall\,\tau\in\{t_k,\dots,t_{k+1}-1\}$.
\end{proof}
Next, we bound the movement cost within each phase except the last step where the {\it jump} happens as follows.
\begin{lemma}\label{lem:projdistance} In phase $k$, the movement cost $\sum_{\tau=t_k+1}^{t_{k+1}-1}||x_{\tau}-x_{\tau-1}||$ just before the {\it jump} to $x^\star_{t_{k+1}}$ for algorithm  \textsc{Frugal} is at most $O(2 d^{d/2} \delta_k)$.
\end{lemma} 
\begin{proof} Since for $\tau \in \{t_k+1,\dots,t_{k+1}-1\}$, \textsc{Frugal} plays {\it lazily}: $x_\tau = \ell_\tau = \Pi_{S_\tau}(x_{\tau-1})$, the curve 
$(x_{t_k}, x_{t_k+1}, x_{t_k+2}, \dots, x_{t_{k+1}-1})$ is a projection curve. Thus, combining, Lemma \ref{lem:projdistanceint} and Lemma \ref{lem:phase-confinement}, we get the result.
\end{proof}

\begin{proof}(Proof of Theorem \ref{lem:mcBalanceStrongPhase})
The movement cost of the last step in phase $k$ when the algorithm {\it jumps} from $x_{t_{k+1}-1}$ to $x_{t_{k+1}}=x_{t_{k+1}^\star}$ is $||x_{t_{k+1}^\star}-x_{t_{k+1}-1}|| \le \delta_k$. Combining this with Lemma \ref{lem:projdistance}, we get the result of Lemma \ref{lem:mcBalanceStrongPhase}.
\end{proof}

Next, we complete the proof of Theorem \ref{thm:rsb_sc}.
\begin{proof}(Proof of Theorem \ref{thm:rsb_sc})
Using Lemma \ref{lem:mcBalanceStrongHead} and Lemma \ref{lem:mcBalanceStrongPhase}, and summing across all $N$ phases of \textsc{Frugal}, we get 
\begin{align*}
M_{\textsc{Frugal}}(T) & = \sum_{k=1}^{N-1}(\delta_k + M_{k}), \\
& \le O(d^{d/2}) \sum_{k=1}^{N-1}\delta_k, \\
& \le O\left(d^{d/2}\frac{2G+(L+\mu)D}{\mu} \log T\right). \\
\end{align*}
\end{proof}

\section{Proof of Theorem 8}
\label{app:thm:lbsclogT}
Let $f(x) = \frac{1}{2}\|x\|^2$, which is $1$-strongly convex and $1$-smooth. Since $f$ is strongly convex, we recall that $x_t^\star= \arg \min_{x\in S_t} f(x)$ and $v_t=f(x_t^\star)$. Let $\mathcal{X} \subset \mathbb{R}^2$ be a rectangle $$\mathcal{X} = \{(p,q) : |p| \le \frac{1}{2},\; -b \le q \le -a\}$$ for constants $a, b > 0$ with $b > a$ and $a > \frac{1}{4(b-a)}$. 

Define $$R_{\max} := \max_{x \in \mathcal{X}}\|x\| \le \sqrt{\frac{1}{4}+b^2}$$ to be the maximum distance of any point of $\mathcal{X}$ from the origin, which is bounded by a constant $\Theta(1)$. To derive the lower bound for any online algorithm $\mathcal{A}$, we consider the input revealed over $K$ periods. 

We define the target number of periods dynamically based on the horizon $T$:$$K := \left\lfloor c \frac{\log T}{\log \log T} \right\rfloor$$ for a sufficiently small constant $c \in (0, 1)$ to be determined later. Set the spatial parameters to scale with $K$:$$a > \frac{B^2}{4\delta}, B := \frac{1}{\sqrt{K}} \qquad \text{and} \qquad \delta := \frac{b-a}{K}.$$Note that the nesting condition $a > \frac{B^2}{4\delta}$ evaluates to $a > \frac{1/K}{4(b-a)/K} = \frac{1}{4(b-a)}$, which is strictly satisfied by our choice of constants $a$ and $b$, independently of $K$. Fix the tracking tolerance $\epsilon := \frac{B}{4} = \frac{1}{4\sqrt{K}}$. Define the regret gap parameter:$$\eta := \frac{\epsilon^2}{2} = \frac{1}{32K}.$$ Define period minimizers for $k = 1,\ldots,K$:$$x_k^\star := \left((-1)^k\frac{B}{2\sqrt{2}},\;-(a+k\delta)\right).$$$x_k^\star \in \mathcal{X}$ for $k = 1,\ldots,K$ since $|p_k^\star| = \frac{B}{2\sqrt{2}} \le \frac{1}{2}$ and $q_k^\star = -(a+k\delta) \in [-(a+b-a),-a] = [-b,-a]$. Set $x_0 := x_1^\star, S_0 := \mathcal{X}, t_1 := 1$.Let $$H_k := \{x \in \mathbb{R}^2 : \langle x_k^\star, x - x_k^\star\rangle \ge 0\}.$$ Let the action of $\mathcal{A}$ at time $t$ be $x_t$. At time $t_k$, the adversary reveals $S_{t_k} = S_{t_{k-1}} \cap H_k$ and holds it constant until the period invariant $\|x_t - x_k^\star\| \le \epsilon$ is first met by $\mathcal{A}$, whereupon period $k+1$ begins. Thus, $S_\tau = S_{t_k}$ for $\tau=t_k, t_k+1, \dots, t_{k+1}-1$, where $t_{k+1}$ is the earliest time $t$ at which $\mathcal{A}$ chooses an action $x_t$ such that $\|x_t - x_k^\star\| \le \epsilon$. 

\begin{lemma}\label{lem:a1} For all $k=1,\ldots,K$, $x_k^\star \in S_{t_k}$ is the unique minimizer of $f$ over $S_{t_k}$, and $v_k := \min_{x\in S_{t_k}}f(x) =  \frac{1}{2}\|x_k^\star\|^2$.
\end{lemma}
\begin{proof} The fact that $x_k^\star \in S_{t_k}$ follows from $S_{t_k} = S_{t_{k-1}} \cap H_k$.The identity $\|x\|^2 = \|x_k^\star\|^2 + 2\langle x_k^\star, x-x_k^\star\rangle + \|x-x_k^\star\|^2$ gives: $$f(x) - v_k = \langle x_k^\star, x-x_k^\star\rangle + \frac{1}{2}\|x-x_k^\star\|^2.$$For $x \in S_{t_k} \subseteq H_k$, $\langle x_k^\star, x-x_k^\star\rangle \ge 0$, so $f(x) - v_k \ge \frac{1}{2}\|x-x_k^\star\|^2 \ge 0$, with equality iff $x = x_k^\star$.
\end{proof}
\begin{lemma}\label{lem:a2} For any $t \in [t_k, t_{k+1}-1]$, $f(x_t) - v_k \ge \eta > 0$.\end{lemma}

\begin{proof} From Lemma \ref{lem:a1}, while the invariant is unmet, $f(x_t) - v_k \ge \frac{1}{2}\|x_t-x_k^\star\|^2 > \frac{1}{2}\epsilon^2 = \eta$. 
\end{proof}
Fix any period $k$. During period $k$, the feasible set is frozen, so $v_t = v_k$. By feasibility, $f(x_t) \ge v_t$, and since $S_t\subseteq S_{t-1}$ for all $t$, the sequence $v_t=\min_{x\in S_t}f(x)$ is nondecreasing in $t$, so $v_t\ge v_1$. Hence $f(x_t) \ge v_t \ge v_1$. Since $v_k \le v_K$, defining $\Delta v_{\max} := v_K-v_1$, we have $f(x_\tau) - v_k \ge -\Delta v_{\max}$ for all $\tau < t_k$. Notice that $\Delta v_{\max} \le \frac{1}{2}R_{\max}^2 = \Theta(1)$. Hence, \begin{equation}\label{eq:a3}
\sum_{\tau=1}^{t_k-1}(f(x_\tau)-v_k) \ge -(t_k-1)\Delta v_{\max}.
\end{equation}

 Expanding the regret for 
\begin{align}
\mathrm{Regret}_\mathcal{A}(t) & = \sum_{\tau=1}^t (f(x_\tau)  - v_k), \\
& = \sum_{\tau=1}^{t_k-1} (f(x_\tau)  - v_k) + \sum_{\tau=t_k}^t (f(x_\tau)  - v_k),\\ \label{eq:a4}
& \ge -(t_k-1)\Delta v_{\max} + \eta(t-t_k+1),
\end{align}
where the last inequality is obtained by combining Lemma \ref{lem:a2} and \eqref{eq:a3}.

For any $t \in [t_k, t_{k+1}-1]$, the regret is bounded by $c_R t^\lambda$. Thus, \eqref{eq:a4} forces the inequality: $$\eta(t-t_k+1) \le (t_k-1)\Delta v_{\max} + c_R t^\lambda.$$ Since $\lambda < 1$, the linear left side eventually dominates the sublinear right side, guaranteeing each period terminates in finite time. Applying this bound at $t=t_{k+1}-1$, we obtain: 
\begin{equation}\label{eq:recurrence1}
\eta t_{k+1} \le t_k(\eta + \Delta v_{\max}) + c_R t_{k+1}^\lambda.
\end{equation}

To decouple $t_{k+1}$, we bound the sublinear term globally. For $\lambda \in (0, 1)$, by the concavity of $t^\lambda$ on $[0, \infty)$, its tangent line at any $\theta > 0$ strictly dominates the curve everywhere: 
\begin{equation}\label{eq:tangent}
t^\lambda \le \lambda \theta^{\lambda-1} t + (1-\lambda)\theta^\lambda \quad \forall t \ge 0.
\end{equation}

Multiplying by $c_R$, we choose $\theta$ such that the linear coefficient exactly matches $\frac{\eta}{2}$:
$$c_R \lambda \theta^{\lambda-1} = \frac{\eta}{2} \implies \theta = \left(\frac{2 c_R \lambda}{\eta}\right)^{\frac{1}{1-\lambda}}.$$ 

Substituting this $\theta$ back into \eqref{eq:tangent} gives:
$$c_R t^\lambda \le \frac{\eta}{2} t + C_\lambda, \qquad \text{where } C_\lambda := c_R(1-\lambda)\left(\frac{2 c_R \lambda}{\eta}\right)^{\frac{\lambda}{1-\lambda}}.$$
For the boundary case $\lambda = 0$, we trivially have $c_R t^0 = c_R \le \frac{\eta}{2} t + c_R$, so we simply define $C_0 := c_R$. In both cases, since $\eta = \Theta(1/K)$, we obtain $C_\lambda = \Theta(K^{\frac{\lambda}{1-\lambda}})$.
%
%
%

Substituting this  into our phase length inequality \eqref{eq:recurrence1}, $$\eta t_{k+1} \le t_k(\eta + \Delta v_{\max}) + \frac{\eta}{2} t_{k+1} + C_\lambda.$$Subtracting $\frac{\eta}{2} t_{k+1}$ and multiplying by $\frac{2}{\eta}$ yields the following affine recurrence valid for all $k \ge 1$: 
\begin{equation}\label{eq:recurrence2}
t_{k+1} \le \gamma t_k + \beta,
\end{equation}
where $\gamma := 2 + \frac{2\Delta v_{\max}}{\eta} = \Theta(K)$ and $\beta := \frac{2 C_\lambda}{\eta} = \Theta(K^{\frac{1}{1-\lambda}})$. 

We prove by induction that for any $k \ge 1$, for \eqref{eq:recurrence2}, $$t_k \le \gamma^{k-1} t_1 + \beta \frac{\gamma^{k-1}-1}{\gamma-1}.$$ 

The base case $k=1$ holds trivially as $t_1 \le t_1 + 0$. Assume the bound holds for some $k$. Then for $k+1$:$$t_{k+1} \le \gamma t_k + \beta \le \gamma \left( \gamma^{k-1} t_1 + \beta \frac{\gamma^{k-1}-1}{\gamma-1} \right) + \beta = \gamma^k t_1 + \beta \frac{\gamma^k - \gamma + \gamma - 1}{\gamma-1} = \gamma^k t_1 + \beta \frac{\gamma^k-1}{\gamma-1}.$$ This completes the induction. Since $t_1 = 1$ and $\gamma > 2$, we have $\frac{\gamma^{K-1}-1}{\gamma-1} \le \gamma^{K-1}$, giving the upper bound:$$t_K \le \gamma^{K-1} (1 + \beta).$$Taking the natural logarithm:$$\log t_K \le (K-1)\log \gamma + \log(1+\beta).$$ Substituting the asymptotic bounds $\gamma = \Theta(K)$ and $\beta = \Theta(K^{\frac{1}{1-\lambda}})$, we get $$\log t_K \le K \log \Theta(K) + \log \Theta(K^{\frac{1}{1-\lambda}}) = K \log K + O(K) + O(\log K).$$

By definition, $K = \Theta\left(\frac{\log T}{\log \log T}\right)$. 
Substituting $K$ into the dominant $K \log K$ term:$$K \log K \le c \frac{\log T}{\log \log T} \log\left( c \frac{\log T}{\log \log T} \right) = c \log T + c \frac{\log T}{\log \log T}(\log c - \log \log \log T) = c \log T + o(\log T).$$Combining these limits, we establish:$$\log t_K \le c \log T + o(\log T).$$Therefore, for any fixed $c < 1$, we have $\log t_K < \log T$ for all sufficiently large $T$. This rigorously proves $t_K \le T$. Hence, all $K$ periods are completed before the horizon. Finally, we compute the movement cost contributed by successfully tracking the minimizers. For periods $1 < k < K$, period $k-1$ ends with $x_{t_k-1} \in \mathcal{B}(x_{k-1}^\star,\epsilon)$ and period $k$ ends with $x_{t_{k+1}-1} \in \mathcal{B}(x_k^\star,\epsilon)$. The minimum distance traversed during period $k$ is:$$\|x_k^\star - x_{k-1}^\star\| - 2\epsilon \ge \frac{B}{\sqrt{2}} - 2\epsilon = \frac{1}{\sqrt{2K}} - \frac{1}{2\sqrt{K}} = \frac{\sqrt{2}-1}{2\sqrt{K}} = \Omega\left(\frac{1}{\sqrt{K}}\right).$$Summing this forced movement over the $K-2$ successfully completed intermediate periods:$$M_\mathcal{A}(T) \ge (K-2)\cdot \Omega\left(\frac{1}{\sqrt{K}}\right) = \Omega(\sqrt{K}) = \Omega\left(\sqrt{\frac{\log T}{\log \log T}}\right).$$

\section{Proof of Theorem \ref{thm:lbGconvex}}\label{app:thm:lbGconvex}
Since non-strongly convex costs can have non-unique minimizers, there could be multiple minimizers in any given constraint set. Therefore, for each set $S_t$, we consider the minimizer subset $\mathbf{X}_t^* := \arg\min_{x\in S_t}f(x)$ and extend our focus to a class of greedy algorithms that at timestep $t$ choose some $x_t \in \mathbf{X}_t^*$. Note that all these algorithms will still achieve non-positive regret.
\begin{definition}
    When the minimizer of $f$ in a constraint set is not unique, there can exist many instantiations of the greedy algorithm that achieve identical regret but vary in their movement costs. We define $\cG({\{y_t\}_{t=1}^T})$ as the instance of the greedy algorithm that plays actions according to the sequence $\{y_t\}_{t=1}^T$ ($y_t \in \mathbf{X}_t^*$ for all $ t\in[1,T]$).
\end{definition}

\begin{proof}
For this proof, we use the input with i) $d=2$, 
ii) \textbf{Admissible set} $\mathcal{X}\subseteq\mathbb{R}^2$ is a square that does not contain the origin, and iii)
\textbf{Cost function} $f(x,y) = \max(|x|,|y|)$.
Note that the cost function $f$ is convex but neither smooth nor strongly convex. 

The first few steps of the proposed nesting scheme are shown in Figure \ref{fig:lb_gen}. The construction is made such that $f$ has a unique minimize in each set $S_t$. 
Thus the greedy algorithm $\cG$ is uniquely defined.

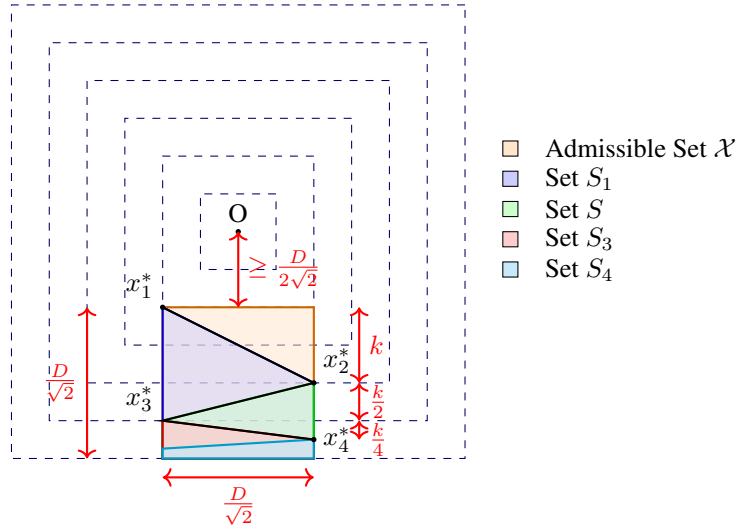
\begin{figure}
    \centering
    \begin{tikzpicture}[scale=1]
    \foreach \r in {0.5, 1, 1.5, 2, 2.5, 3} {
        \draw[blue!40!black, dashed] (\r, \r) -- (-\r, \r) -- (-\r,-\r) -- (\r,-\r) -- cycle;
    }

    \filldraw[fill=orange!20, fill opacity=0.5, draw=orange!80!black, thick]
        (-1, -1) rectangle (1, -3);
    \fill (0,0) circle[radius=1pt];
    \node[above] at (0,0) {O};
    \filldraw[fill=blue!20, fill opacity=0.5, draw=blue!80!black, thick] (-1, -1) -- (1, -2) -- (1, -3) -- (-1, -3)-- cycle;
    \fill (-1, -1) circle[radius=1pt];
    \node[above left] at (-1, -1) {$x_1^*$};
    \filldraw[fill=green!20, fill opacity=0.5, draw=green!80!black, thick] (-1, -2.5) -- (1, -2)-- (1, -3) -- (-1, -3)-- cycle;
    \fill (1, -2) circle[radius=1pt];
    \node[above right] at (1, -2) {$x_2^*$};
    \filldraw[fill=red!20, fill opacity=0.7, draw=red!80!black, thick] (-1, -2.5) -- (1, -2.75)--(1, -3) --  (-1,-3)-- cycle;
    
    \node[above left] at (-1, -2.5) {$x_3^*$};
    \filldraw[fill=cyan!20, fill opacity=0.7, draw=cyan!80!black, thick] (-1, -2.87) -- (1, -2.75)--(1, -3) --  (-1,-3)-- cycle;
     \fill (1, -2.75) circle[radius=1pt];
    \node[right] at (1, -2.75) {$x_4^*$};
      \draw[black, thick] (-1, -1) -- (1, -2) -- (-1, -2.5) -- (1, -2.75);
      \begin{scope}[shift={(3.5,1)}] 
    \draw[fill=orange!20] (0,0) rectangle (0.2,0.2);
    \node[right=1pt] at (0.4,0.1) {Admissible Set $\mathcal{X}$};
    \draw[fill=blue!20] (0,-0.4) rectangle (0.2,-0.2);
    \node[right=1pt] at (0.4,-0.3) {Set $S_1$};
    \draw[fill=green!20] (0,-0.8) rectangle (0.2,-0.6);
    \node[right=1pt] at (0.4,-0.7) {Set $S$};
    \draw[fill=red!20] (0,-1.2) rectangle (0.2,-1);
    \node[right=1pt] at (0.4,-1.1) {Set $S_3$};
    \draw[fill=cyan!20] (0,-1.6) rectangle (0.2,-1.4);
    \node[right=1pt] at (0.4,-1.5) {Set $S_4$};
  \end{scope}
  \draw[<->, thick, red] (1.6, -1) -- (1.6, -2) node[midway, right] {$k$};
  \draw[<->, thick, red] (1.6, -2) -- (1.6, -2.5) node[midway, right] {$\frac{k}{2}$};
  \draw[<->, thick, red] (1.6, -2.5) -- (1.6, -2.75) node[below, right] {$\frac{k}{4}$};
  \draw[<->, thick, red] (-2, -1) -- (-2, -3) node[midway, left] {$\frac{D}{\sqrt{2}}$};
  \draw[<->, thick, red] (-1, -3.25)--(1, -3.25) node[midway, below] {$\frac{D}{\sqrt{2}}$};
  \draw[<->, thick, red] (0, 0)--(0, -1);
  \node[right, red] at (0, -0.5) {$\ge\frac{D}{2\sqrt{2}}$};
\end{tikzpicture}
    \caption{Nesting Scheme}
    \label{fig:lb_gen}
\end{figure}
Formally, let us consider $\cX$ to be the following square:
\[
\cX = \left\{ (p,q) : \frac{-D}{2\sqrt2}\le p\le \frac{D}{2\sqrt2}, -a-\frac{D}{\sqrt2}\le q\le -a\right\}
\]
for some $a\ge D/2\sqrt2$. Note that the contours of the function are squares centered at the origin with their sides parallel to the coordinate axes.
Fixing $a\ge D/2\sqrt2$ ensures that the sides of $\cX$ that are parallel to the x-axis have the same function value (The side on the top will have $f(p,q) = \max(|p|,a) = a$ and the one on the bottom will have $f(p,q) = \max(|p|,a+D/\sqrt2) = a+D/\sqrt2$).\\

The desired set of minimizers $\{x_t^*\}$ for this construction of $S_t$'s are defined as:
\begin{equation}\label{defn:minimizers}
x_t^* = (p_t^*, q_t^*)=\begin{cases}
    (-D/2\sqrt2, -a  -\sum_{i=1}^t k/2^{t-1}) \quad \text{if $t$ is odd,}\\
    (+D/2\sqrt2, -a  -\sum_{i=1}^t k/2^{t-1}) \quad \text{if $t$ is even.}
\end{cases}
\end{equation}
The idea is to make the greedy algorithm $\cG$ oscillate between the vertical sides of $\cX$, while incrementally moving down in the vertical direction. Note that at timestep $t$, for the minimizer $x_t^*$ to be in the feasible set, the value of $q_t^*$ must satisfy,
$-a-\frac{D}{\sqrt{2}}\le q_t^*\le -a$.

The value of $p_t^*$ is inconsequential, as its role is to keep alternating between the vertical edges of $\cX$ ($p_t^*= \pm D/2\sqrt2$ for all $t$). Therefore, if we set the value of $k\le D/2\sqrt2$,
\[
\lim_{t\to\infty}q_t^* = -a-\sum_{i=1}^\infty\frac{k}{2^{i-1}} \ge
-a-\frac{D}{\sqrt2},
\]
this implies that the nesting scheme can be run indefinitely before hitting the bottom side of the square (intersecting with the line $q=-a-D/\sqrt{2}$).

To construct a set that has a unique minimizer at $x_t^*$ \eqref{defn:minimizers}, we consider the line $L_t$ joining the minimizers $x_t^*$ and $x_{t+1}^*$ (as we know all the minimizers $\{x_t^*\}$ needed for our construction). Let $H_t$ be the halfspace with $L_t$ as the boundary that does not contain the origin. Then, the set $S_t$ can be constructed as $S_t := S_{t-1}\cap H_t.$

Since at each step $t$, $\Vert x^*_t -x^*_{t-1}\Vert = \sqrt{\left(\frac{D}{\sqrt{2}}\right)^2+\left(\frac{k}{2^{t-1}}\right)^2} \geq \frac{D}{\sqrt{2}}$, it follows that,
    $$M_\cG{([2:T])}\geq \frac{D(T-1)}{\sqrt{2}} = \Omega(T).$$
\end{proof}
\section{Proof of Theorem \ref{thm:steiner-convex}}\label{app:thm:steiner-convex}
\begin{proof}
We first bound the regret.
Fix any time step $t$ in phase $p$ of the \textsc{LSP} algorithm with start time $\tau_p$.
The algorithm plays $x_t = \Pi_{\Phi_t^{(p)}}(x_{t-1}) \in \Phi_t^{(p)}$. Thus,
 $f(x_t) \le L_p + \varepsilon$.
Since $S_T \subseteq S_{\tau_p}$, we have $L_p \le f(x^\star_T)$, giving
$f(x_t) - f(x^\star_T) \le \varepsilon$.
This holds at every step including phase-transition steps.
Summing: $$\mathrm{Regret}_{\textsc{LSP}}(T) = \sum_{t=1}^T f(x_t) - f(x^\star_T)\le T\varepsilon.$$

Next, we move on to upper bounding the movement cost.

\textbf{Step A: Phase count.}
Since $f$ is $G$-Lipschitz on $S_1$ (diameter $D$), all function values on $S_1$
lie in an interval of width $\le GD =: \Delta$.
Each phase transition raises the level by more than $\varepsilon$. Thus, if  $N$ is the total number of phases encountered by algorithm \textsc{LSP}, we have
$$N \le \left\lfloor\frac{GD}{\varepsilon}\right\rfloor + 1.$$

\textbf{Step B: Active set diameter.}
During phase $p$, every point played by the algorithm lies in
$\Phi_t^{(p)} \subseteq S_t \subseteq S_1$.
Since $S_1$ has diameter $D$:
\begin{equation}\label{eq:diam-convex}
\mathrm{diam}(\Phi_t^{(p)}) \le D \qquad \text{for all } t, p.
\end{equation}

\textbf{Step C: Within-phase movement.}
During phase $p$, the level $L_p$ is fixed, so $\Phi_{t+1}^{(p)} \subseteq \Phi_t^{(p)}$
(since $S_{t+1} \subseteq S_t$). The sequence $\{\Phi_t^{(p)}\}_{t=\tau_p}^{\tau_{p+1}-1}$
is therefore nested, with initial diameter $\le D$ by \eqref{eq:diam-convex}.
Using Lemma \ref{lem:projdistanceint}, the movement cost within each phase is 
$$
\le O (d^{d/2} \cdot D).$$
Summing over all $N$ phases:
\begin{equation}\label{eq:within-convex}
M_{\textsc{LSP}}^{\mathrm{within-phase}} \le O (d^{d/2} \cdot D N).
\end{equation}

\textbf{Step D: Between-phase movement.}
At each of the $N-1$ phase transitions, the algorithm moves from
$x_{\mathrm{end}} = \Pi_{(\Phi_{\tau_{p+1}-1}^{(p)})}(x_{\mathrm{end}}-1)$ to
$x_{\mathrm{start}} = \Pi_{(\Phi_{\tau_{p+1}}^{(p+1)})}(x_{\mathrm{end}})$.
Both points lie in $S_1$, so:
$$\|x_{\mathrm{end}} - x_{\mathrm{start}}\| \le D.$$
Summing over all $N-1$ transitions:
\begin{equation}\label{eq:between-convex}
M_{\textsc{LSP}}^{\mathrm{in-between \ phases}} \le (N-1)D \le ND.
\end{equation}

\textbf{Step E: Combining the bounds.}
Adding \eqref{eq:within-convex} and \eqref{eq:between-convex}:
$$M_{\textsc{LSP}}(T) = M_{\textsc{LSP}}^{\mathrm{within-phase}} + M_{\textsc{LSP}}^{\mathrm{in-between \ phases}}
\le O (d^{d/2} \cdot D N) + ND = O (d^{d/2} \cdot D N).$$
Substituting $N \le GD/\varepsilon + 1$:
Hence $M_{\textsc{LSP}}(T)= O\!\left(\dfrac{d^{d/2}GD^2}{\varepsilon}\right)$.
\end{proof}

\section{Use-Case for Directional Input}\label{app:proginput}
\paragraph{Margin-Constrained Classification under Tightening Margin Requirements.}
Margin-constrained binary classification Section~\ref{sec:rlhf}) provides a natural and 
geometrically clean instance of directional input (Definition~\ref{def:directional-interval}))  in $\mathbb{R}^2$,
in the spirit of large-margin learning \cite{cortes1995support, bartlett1998sample}.

\paragraph{Setup.}
Let $d=2$, admissible set $\cX = [0,D]^2$, and objective
\[
    f(x_1, x_2) \;=\; x_1 + x_2 + \varepsilon x_1^2,
    \qquad
    \varepsilon \;=\; \frac{4}{(D-2)^2},
\]
which is convex for all $\varepsilon \ge 0$ (the Hessian has eigenvalues $2\varepsilon \ge 0$ 
and $0$, so $f$ is convex but \emph{not} strongly convex). Here $x_1$ represents the 
average margin shortfall of the classifier and $x_2$ represents a regularization 
penalty, with the term $\varepsilon x_1^2$ encoding a mild convex penalty on margin 
violations consistent with margin-based generalization bounds \cite{bartlett1998sample}. 
The quadratic term breaks the symmetry between $x_1$ and $x_2$, ensuring the minimizer 
is unique and sits at $x_1=0$. At round $t$, the margin requirement tightens, restricting 
the classifier to
\[
    S_t \;=\; \bigl\{(x_1,x_2)\in[0,D]^2 : x_1 + x_2 \;\ge\; t\bigr\},
\]
for $t\le D$, and  $S_1 \supsetneq S_2 \supsetneq \cdots \supsetneq S_T$. The constraint $x_1+x_2 \ge t$ 
encodes a minimum total resource commitment at round $t$: both the margin shortfall and 
regularization cost must jointly exceed a growing floor, reflecting the fact that as 
margin requirements tighten, the classifier is forced into regions of higher combined cost 
\cite{cortes1995support}. The nested structure $S_t \subseteq S_{t-1}$ mirrors the CONES 
framework exactly.

\paragraph{Local minimizers.}
On the boundary $x_1+x_2=t$, substituting $x_2=t-x_1$ gives $f = t + \varepsilon x_1^2$,
minimized at $x_1=0$. Hence the unique minimizer of $f$ over $S_t$ is
\[
    x^\star_t \;=\; (0,\, t),
    \qquad
    v_t \;=\; t.
\]

\paragraph{Projection sequence.}
Fix any $s < t$. Initialize at $y_s = x^\star_s = (0,s) \in X^*_s$.
For $\tau = s+1,\ldots,t-1$, since $y_{\tau-1,1}+y_{\tau-1,2}=\tau-1 < \tau$, 
the point $y_{\tau-1}$ lies outside $S_\tau$ and the Euclidean projection onto 
$\{x_1+x_2 \ge \tau\}$ adds $\delta_\tau = 1/2$ to each coordinate:
\[
    y_\tau \;=\; \Bigl(\frac{\tau-s}{2},\;\; s+\frac{\tau-s}{2}\Bigr)
    \quad\forall\,\tau \in \{s,s+1,\ldots,t-1\}.
\]
One verifies $y_{\tau,1}+y_{\tau,2}=\tau$, so $y_\tau \in \partial S_\tau$. The unique 
minimizer is $x^\star_\tau=(0,\tau)$, while $y_{\tau,1}=(\tau-s)/2 > 0$ for all $\tau > s$, 
so
\[
    y_\tau \;\ne\; x^\star_\tau \qquad \forall\,\tau \in \{s+1,\ldots,t-1\}.
\]
Intermediate projections are \emph{strictly away} from local minimizers: 
the equal-split projection accumulates margin shortfall $(\tau-s)/2 > 0$, 
whereas the true optimum always sits at zero margin shortfall.

\paragraph{Checking Definition~\ref{def:directional-interval}.}
Computing $f$ along the projection sequence:
\[
    f(y_\tau) 
    \;=\; \frac{\tau-s}{2} + s + \frac{\tau-s}{2} + \varepsilon\Bigl(\frac{\tau-s}{2}\Bigr)^2
    \;=\; \tau + \frac{\varepsilon(\tau-s)^2}{4}.
\]
We require $f(y_\tau) \le v_t = t$ for all $\tau \in \{s+1,\ldots,t-1\}$, i.e.
\[
    \frac{\varepsilon(\tau-s)^2}{4} \;\le\; t - \tau. \tag{$*$}
\]
Setting $u=\tau-s \in \{1,\ldots,t-s-1\}$, condition~$(*)$ becomes 
$\varepsilon \le g(u) := 4(t-s-u)/u^2$. Since $g'(u) = -4(2(t-s)-u)/u^3 < 0$ 
for $u < 2(t-s)$ (which holds on $\{1,\ldots,t-s-1\}$), $g$ is strictly 
decreasing and its minimum is at $u = t-s-1$:
\[
    \min_u g(u) \;=\; \frac{4}{(t-s-1)^2}.
\]
The binding constraint over all $1 \le s < t \le D$ (with $t \ge s+2$ so that 
intermediate steps exist) is at $s=1, t=T$, which gives the smallest value 
$4/(D-2)^2$, since $(t-s-1)^2$ is maximized at $s=1, t=T$. With 
$\varepsilon = 4/(D-2)^2$, condition~$(*)$ holds for all $\tau \in \{s+1,\ldots,t-1\}$ 
and all $s < t$, confirming
\[
    f(y_\tau) \;\le\; v_t = t 
    \qquad \forall\,\tau \in \{s+1,\ldots,t-1\},\;\forall\,s < t,
\]
so $\{S_t\}$ is directional with respect to $f$ (Definition~\ref{def:directional-interval}). 

\paragraph{Why $\varepsilon = 4/(D-2)^2$ is tight.}
The worst-case pair is $(s,t) = (1,D)$ because $(t-s-1)^2$ grows with $t-s$, 
making $4/(t-s-1)^2$ smallest for the largest gap. Any $\varepsilon$ strictly 
larger than $4/(D-2)^2$ would violate~$(*)$ at $\tau = D-1$, $s=1$, $t=D$. 
Note that $\varepsilon \to 0$ as $D \to \infty$: for large horizons, an 
arbitrarily small quadratic penalty suffices to break symmetry and ensure a 
unique minimizer, while keeping $f$ essentially linear (convex but not 
strongly convex) throughout.

\section{Proof of Theorem \ref{thm:directional}}\label{app:thm:directional}
Definition~\ref{def:directional-interval} is stricter than the following mild condition (algorithm dependent) that we actually need to prove Theorem \ref{thm:directional}.
\begin{lemma}\label{lem:relaxed-lazy-bound}
If input is directional (Definition~\ref{def:directional-interval}), then for phase $k-1$, i.e., time steps $(t_{k-1}, t_k]$ where  algorithm \textsc{\textsc{Gap-Frugal}} is {\it lazy}, for all $\tau \in \{t_{k-1}+1, \dots, t_k-1\}$, the iterates satisfy 
\begin{equation}\label{eq:directional}
f(\hat{x}_\tau) \le v_{t_k}.
\end{equation}
\end{lemma}

\begin{proof}
Set $s = t_{k-1}$ and $t = t_k$.

Since $t_{k-1}$ is a jump time of \textsc{Gap-Frugal}, the algorithm plays
\[
    x_s \;=\; x^{\star}_{\mathrm{near},\,s} \;=\; \Pi_{X^{\star}_s}(x_{s-1}).
\]

For each $\tau \in \{s+1, \dots, t-1\}$, since the entire phase $(s, t]$ is {\it lazy}, 
\textsc{Gap-Frugal} plays
\[
    x_\tau \;=\; \hat{x}_\tau \;=\; \Pi_{S_\tau}(x_{\tau-1}).
\]
Thus the sequence $\{x_\tau\}_{\tau=s}^{t-1}$ satisfies
\[
    x_s \in X^{*}_s,
    \qquad
    x_\tau = \Pi_{S_\tau}(x_{\tau-1})
    \quad \text{for } \tau = s+1, \dots, t-1.
\]
This is precisely the sequence $\{y_\tau\}$ appearing in Definition~\ref{def:directional-interval} 
(Directional Input), initialized at the optimizer $y_s = x_s \in X^{*}_s$ and 
advanced by successive Euclidean projections $y_\tau = \Pi_{S_\tau}(y_{\tau-1})$. 
Since the input is directional with respect to $f$, Definition~\ref{def:directional-interval} 
guarantees that
\[
    f(y_\tau) \;\le\; v_{t_k} \;=\; \min_{x \in S_{t_k}} f(x) 
    \qquad
    \text{for all } \tau \in \{s+1, \dots, t-1\}.
\]
Since $x_\tau = \hat{x}_\tau = y_\tau$ throughout the lazy phase, we conclude 
$f(\hat{x}_\tau) \le v_{t_k}$ for all $\tau \in \{t_{k-1}+1, \dots, t_k-1\}$.
\end{proof}

\begin{proof}(Theorem \ref{thm:directional})
We show in \eqref{eq:dummyx1}, that for each {\it jump} time $t_k$, $F_{t_k} \le t_k v_{t_k} + C$. Thus, if $T$ is a {\it jump} time, i.e. $T=t_k$ for some $k$, then we get that $F_T \le Tv_{T}+C$. Otherwise, by 
definition of the \textsc{\textsc{Gap-Frugal}} at time $t=T$, $F_{t-1} + f(\hat{x}_t) \le t\,v_t + C$, and again we have $F_T=F_{t-1} + f(\hat{x}_t) \le Tv_{T}+C$.
which implies that $F_T \le Tv_{T}+C$. Thus, 
$$\mathrm{Regret}_{\textsc{\textsc{Gap-Frugal}}}(T) = \sum_{t=1}^T f(x_t) - T f(x_T^\star) = F_T - T f(x_T^\star) \le Tv_{T}+C - Tv_{T}\le C.$$

For upper bounding the movement cost, in Lemma \ref{lem:jumpprog}, we show that with {\it directional} input the  number of {\it jumps} $N$ encountered by algorithm \textsc{\textsc{Gap-Frugal}} is $O(\log^2 T)$. Since in each phase (between two successive jumps) algorithm \textsc{\textsc{Gap-Frugal}} takes successive projections 
on to nested convex sets, from Lemma \ref{lem:projdistanceint} we have that the movement cost in each phase is at most $O(d^{d/2}D)$. Putting this together, we get that the total movement cost of algorithm \textsc{Gap-Frugal} with {\it directional} input is at most  $O(d^{d/2}D \log^2 T)$.
\end{proof}

\begin{lemma}\label{lem:jumpprog} For  {\it directional} input, the number of {\it jumps} $N$ encountered by algorithm \textsc{\textsc{Gap-Frugal}} is $O(\log^2 T)$.
\end{lemma}
\begin{proof} 
Let $C_t = \sum_{\tau=1}^t \epsilon_\tau$ be the running cumulative threshold at time $t$. We establish by induction on the time steps that at any {\it jump} time $t_k$, the cumulative loss satisfies $$F_{t_k} \le t_k v_{t_k} + C_{t_k}.$$ Base Case ($t_1$): At the first {\it jump}, $x_{t_1} = x_{t_1}^\star$, so $F_{t_1} = v_{t_1}$. Since $C_{t_1} \ge 0$, it trivially holds that $v_{t_1} \le t_1 v_{t_1} + C_{t_1}$. Inductive Step: Assume the bound holds for all {\it jump} times prior to $t_k$. Consider the time step $t_k - 1$. There are two possibilities for \textsc{Gap-Frugal}'s behavior at $t_k - 1$: 

i) It played a {\it lazy} action: This means the Regret Condition passed, giving $F_{t_k-2} + f(\hat{x}_{t_k-1}) \le (t_k - 1)v_{t_k-1} + C_{t_k-1}$. By definition, the LHS is exactly $F_{t_k-1}$. 

ii) It made a {\it jump}: By our inductive hypothesis, we directly have $F_{t_k-1} \le (t_k-1)v_{t_k-1} + C_{t_k-1}$. In either case, prior to the {\it jump} at $t_k$, we have the bound:$$F_{t_k-1} \le (t_k-1)v_{t_k-1} + C_{t_k-1}.$$ 

At time $t_k$, a {\it jump} occurs, meaning the algorithm plays $x_{t_k} = x_{t_k}^\star$, incurring cost $f(x_{t_k}) = v_{t_k}$. Updating the cumulative loss yields:$$F_{t_k} = F_{t_k-1} + v_{t_k} \le (t_k-1)v_{t_k-1} + C_{t_k-1} + v_{t_k}.$$ 
Because the feasible sets are nested ($S_{t_k} \subseteq S_{t_k-1}$), the sequence of optimal values is monotonically non-decreasing, implying $v_{t_k-1} \le v_{t_k}$. Substituting this inequality gives:$$F_{t_k} \le (t_k-1)v_{t_k} + C_{t_k-1} + v_{t_k} = t_k v_{t_k} + C_{t_k-1}.$$ Since $\epsilon_{t_k} > 0$, we have $C_{t_k-1} < C_{t_k}$, yielding the final required bound: 
\begin{equation}\label{eq:dummyx1}
F_{t_k} \le t_k v_{t_k} + C_{t_k}.
\end{equation}

Since a {\it jump} happens at time $t_{k+1}$, we must have 
$$ F_{t_k} +  \sum_{\tau=t_k+1}^{t_{k+1}} f(\hat{x}_\tau) > t_{k+1} v_{t_{k+1}} +  \sum_{\tau=1}^{t_{k+1}} \epsilon_\tau,$$
which using \eqref{eq:dummyx1} gives 
\begin{equation}\label{eq:dummyx2}
\sum_{\tau=t_k+1}^{t_{k+1}} f(\hat{x}_\tau) > t_{k+1} v_{t_{k+1}} - t_k v_{t_k} + \sum_{\tau=1}^{t_{k+1}} \epsilon_\tau - C_{t_k}.
\end{equation}

Let $$a_{k+1} = v_{t_{k+1}} - v_{t_k}.$$
Invoking the {\it directional} input condition ($f(\hat{x}_\tau) \le v_{t_{k+2}}$ for all $\tau \in t_{k}, \dots, t_{k+1}$ ) and dropping the positive term $\sum_{\tau=1}^{t_{k+1}} \epsilon_\tau - C_{t_k}$  in \eqref{eq:dummyx2}, we get
 \begin{equation}\label{eq:dummyx3} 
 (t_{k+1} - t_k) a_{k+2} > t_k a_{k+1}.
\end{equation}

Define the potential $Z_k = t_k a_{k+1}$ and the ratio $\rho_k = \frac{t_{k+1}}{t_{k+1}-t_k}$. Multiplying \eqref{eq:dummyx3}  by $\rho_k$ yields: 
\begin{equation}\label{defn:PotZ}
Z_{k+1} > \rho_k Z_k.
\end{equation}

Let $\Delta t_k = t_{k+1}-t_k$.
 We partition the set of {\it jump} indices $\{1, \dots, N\}$ into two sets: $$\mathcal{N}_{time} = \{k : \Delta t_k > t_k\}$$ and 
$$\mathcal{N}_{pot} = \{k : \Delta t_k \le t_k\}.$$ For $k\in \mathcal{N}_{time}$, $t_{k+1} > 2t_k$. Since $t_k \le T$, $|\mathcal{N}_{time}| \le \log T$. 

For $k\in \mathcal{N}_{pot}$, $\Delta t_k \le t_k$ which implies that $\rho_k \ge 2$, so $Z_{k+1} > 2Z_k$ from \eqref{defn:PotZ}.

To upper bound $\mathcal{N}_{pot}$, let the partition of the set of {\it jump} indices be $\{1, \dots, N\} = \{Q_1, P_1, \dots, Q_L, P_L\}$ where for each $j$, all $Q_j \subset \mathcal{N}_{time}$ and  $ P_j\subset \mathcal{N}_{pot}$.
Since $|\mathcal{N}_{time}|\le O(\log T)$, $L\le O(\log T)$. We next show that any $P_j$ can have at most $O(\log T)$ {\it jumps} within it.

For any $k,j$, if $k\in P_j$ then $Z_{k+1} > 2Z_k$. 
Given that $Z_k = t_k (v_{t_{k+1}} - v_{t_k}) \le T \cdot GD$,  thus, the number of {\it jumps} in any $P_j$ is  at most: 
$$\log_2(T GD) = O(\log T).$$ 

Hence $|\mathcal{N}_{pot}|\le O(\log T) |\mathcal{N}_{time}| \le O(\log^2 T)$.
Counting all {\it jump} indices, the total number of jumps $N = |\mathcal{N}_{time}|  + |\mathcal{N}_{pot}| \le O(\log T) + L \ O(\log T) = O(\log^2 T)$.

\end{proof}
\section{Proof of Lemma \ref{lem:projdistanceint} and Lemma \ref{lem:projdistance}}\label{app:projdistance}
%
%
%
%
%


\begin{definition}\label{defn:se-curve} A curve $\gamma: I \rightarrow \bbR^d$  is called self-expanded, if for every $t$ where 
$\gamma'(t)$ exists, we have 
$$< \gamma'(t), \gamma(t)-\gamma(u)> \ \ge 0$$ for all $u\in I$ with $u \le t$, where $<.,.>$ represents the inner product. 
In words, what this means is that $\gamma$ starting in a point $x_0$ is self expanded, if for every $x\in \gamma$ for which there exists the tangent line $\sfT$, the arc (sub-curve) $(x_0, x)$ is
contained in one of the two half-spaces, bounded by the hyperplane through
$x$ and orthogonal to $\sfT$. 
\end{definition}
For self-expanded curves the following classical result is known.
\begin{theorem}\label{thm:manselli}\cite{Manselli}
For any self-expanded curve $\gamma$ belonging to a closed bounded convex set of $\bbR^d$ with diameter $D$, its total length is at most $O(d^{d/2} D)$.
\end{theorem}

\begin{proof}(Proof of Lemma \ref{lem:projdistanceint})
From Definition \ref{defn:projectioncurve}, the projection curve is 
$${\underline \sigma}=\{(\sigma_1,\sigma_2), (\sigma_2,\sigma_3), \dots, (\sigma_{T-1},\sigma_T)\}.$$ Let the reverse curve be ${\underline r} = \{r_t\}_{t=0, \dots, T-2}$, where $r_t = (\sigma_{T-t}, \sigma_{T-t-1})$. Thus we are reading ${\underline \sigma}$ backwards and calling it ${\underline r}$. Note that since $\sigma_{t}$ is the projection of $\sigma_{t-1}$ on $K_t$, each piece-wise linear segment $(\sigma_t, \sigma_{t+1})$ is a straight line and hence differentiable except at the end points. Moreover, since each $\sigma_t$ is obtained by projecting $\sigma_{t-1}$ onto $K_t$ and $K_{t+1}\subseteq K_t$, we have that the projection hyperplane 
$F_t$ that passes through $\sigma_t=\Pi_{K_t}(\sigma_{t-1})$ and is perpendicular to $\sigma_t - \sigma_{t-1}$ separates the two sub curves $\{(\sigma_1,\sigma_2), (\sigma_2,\sigma_3), \dots, (\sigma_{t-1},\sigma_t)\}$ and $\{(\sigma_t,\sigma_{t+1}), (\sigma_{t+1},\sigma_{t+2}), \dots, (\sigma_{T-1},\sigma_T)\}$.

Thus, we have that 
for each segment $r_\tau$, at each point where it is differentiable, the curve $r_1, \dots r_{\tau-1}$ lies on one side of the hyperplane that passes through the point and is perpendicular to $ r_{\tau+1}$. Thus, we conclude that curve ${\underline r}$ is self-expanded.


As a result, Theorem \ref{thm:manselli} implies that the length of ${\underline r}$ or  ${\underline \sigma}$ is at most $$O(d^{d/2} \text{diameter}(K_1)).$$ 
\end{proof}

Next, we complete the proof of  Lemma \ref{lem:projdistance}.
Let $\cB(x,r)$ be a ball of radius $r$ centered at $x$. Then, for proving Lemma \ref{lem:projdistance}, 
we define $K_t = S_{t_{k}+t}\cap \cB(x_k^\star, 2 \delta_k)$ for $t=1, \dots, t_{k+1} - t_{k} -1$. 
From Lemma \ref{lem:phase-confinement}, we know that each $x_{t_{k}+t}$ belongs to $K_1=S_{t_{k}+1}\cap \cB(x_k^\star, 2 \delta)$ and  $x_{t_{k}+t} =  \Pi_{K_{t}}(x_{t_{k}+t-1})$ for $t=1, \dots, t_{k+1} - t_{k} -1$. Thus, applying Lemma \ref{lem:projdistanceint}, we get that the total movement cost $\sum_{\tau=t_k+1}^{t_{k+1}-1}||x_{\tau}-x_{\tau-1}||$ just before the jump to $x_{t_{k+1}*}$ 
is at most $O(2d^{d/2} \delta_k)$

%

\section{Generic Solution for Problem \eqref{eq:lincost} under Dynamic Constraints}\label{sec:ncbclincost} 
\subsection{$\sfp=1$}
As shown in \cite{Sellke2023}, problem \eqref{eq:lincost} when $\sfp=1$ is equivalent to a NCBC problem in $d+1$ dimensions as follows.
For each time step $t \ge 1$, the {\bf work function} $w_t: \mathbb{R}^d \to \mathbb{R}$ is defined recursively:
\begin{equation*}
    w_t(x) = \inf_{y \in \mathbb{R}^d} \left( w_{t-1}(y) + \|x - y\| \right) + f_t(x)
\end{equation*}
with the base case $w_0(x) = 0$ (or the indicator of the initial set).
The problem is lifted to the space $ \mathbb{R}^{d+1}$, where a point is denoted by the pair $(x, h)$ with $x \in \mathbb{R}^d$ and $h \in \mathbb{R}$.
The sequence of convex sets $\{K_t\}_{t \ge 1}$ is defined as the sequence of \textbf{epigraphs} of the work functions:
\begin{equation*}
    K_t = \text{epi}(w_t) = \left\{ (x, h) \in \mathbb{R}^{d+1} \mid w_t(x) \le h \right\}.
\end{equation*}

Due to the non-negativity of the costs $f_t$, the work function is strictly non-decreasing over time ($w_t(x) \ge w_{t-1}(x)$ for all $x$). Consequently, the epigraphs form a nested sequence of convex bodies:
\begin{equation*}
    K_t \subseteq K_{t-1} \subset \mathbb{R}^{d+1}
\end{equation*}
A NCBC chasing algorithm is applied to the sequence $\{K_t\}$ in $\mathbb{R}^{d+1}$ to generate a sequence of points $z_t = (x_t, h_t) \in K_t$. The sequence $x_t \in \mathbb{R}^d$ constitutes the solution to  problem \eqref{eq:lincost} with $\sfp=1$. The challenge, however, is that finding work functions $w(.)$  is challenging, and thus, in prior work, simple algorithms with bounded competitive ratios have been derived for solving \eqref{eq:lincost} when $d=1$ for convex $f_t$'s in \cite{bansal20152}, and for general $d$ when $f_t$'s are strongly convex and smooth \cite{argue2020dimension} when $S_t$'s do not change over time.

The same correspondence can be made even under dynamic constraints by suitably changing the work function definition as follows. 
The work function $w_t: \mathbb{R}^d \to \mathbb{R} \cup \{\infty\}$ is defined recursively as:
\begin{equation*}
    w_t(x) = 
    \begin{cases} 
        \displaystyle \inf_{y \in \text{dom}(w_{t-1})} \left( w_{t-1}(y) + \|x - y\| \right) + f_t(x) & \text{if } x \in S_t \\
        \infty & \text{if } x \notin S_t 
    \end{cases}
\end{equation*}

The body $K_t \subset \mathbb{R}^{d+1}$ is defined as the epigraph of $w_t$ restricted to the feasible cylinder:
\begin{equation*}
    K_t = \text{epi}(w_t) = \left\{ (x, h) \in S_t \times \mathbb{R} \mid w_t(x) \le h \right\}
\end{equation*}
Thus, principally one can map the problem \eqref{eq:lincost} with dynamic constraints  with $\sfp=1$ to a NCBC problem, and use the best known algorithms for solving NCBC to get competitive ratio bound of $O(d+1)$, however, the underlying computational challenges remain.

Thus, in Section \ref{sec:lincostresults}, we consider simple algorithms for solving \eqref{eq:lincost} under dynamic constraints and derive bounds on their competitive ratio results for special cases. In particular, for $d=1$, in Subsection \ref{sec:5comp} we show that a simple algorithm has a  competitive ratio of $5$ for solving problem \eqref{eq:lincost} with dynamic constraints when $f_t$'s are convex. However,  there appears no simple algorithm with bounded competitive ratio for solving \eqref{eq:lincost}  under dynamic constraints with $d\ge 2$  even when $f_t$ are strongly convex and smooth, and remains a challenging open question. Next, in Subsection \ref{sec:alphasharpcr}, we show that if all $f_t$'s are convex and $\alpha$-sharp, then the greedy algorithm $\cG$ that chooses $x_t=x_t^\star=\arg \min_{x\in S_t} f_t(x)$ has a competitive ratio of $\frac{12}{\alpha}$ for solving \eqref{eq:lincost} with dynamic constraints independent of $d$.

\subsection{$\sfp=2$}\label{sec:p2}
With $\sfp=2$, it is known that if $f_t$'s are convex, the competitive ratio for any online algorithm to solve \eqref{eq:lincost} is unbounded \cite{chen2018smoothed}. Thus, the non-trivial regime is when $f_t$'s are $m$-strongly convex.
When there is no constraint i.e. $x_t \in \bbR^d$, greedy online balanced descent (OBD), regularized OBD  \cite{goel2019beyond} and proximal algorithm \cite{zhang2021revisiting} are known to achieve the optimal 
competitive ratio of $O\left(\frac{1}{\sqrt{m}}\right)$. While enforcing constraints $x_t\in S_t$ where $S_t$'s are changing over time, the analysis for only the proximal algorithm (where at time $t$ proximal step is solved while enforcing the constraint $x_t\in S_t$) \cite{zhang2021revisiting} goes through as is and delivers the optimal competitive ratio guarantee.
\section{Proof of Theorem \ref{thm:cr5}}\label{app:proofcr5}
For algorithm $\cA_B$, we start by  noting the following two important properties with respect to $\cC_t(z_t)$.


\begin{lemma}[Topology of Reachable Component]\label{lem:topology_1d}
    Let $a=z_t$. The reachable component $\mathcal{C}_t(a)$ is a closed interval $[L, R]$ such that $L \le a \le R$.
    Consequently, if the algorithm chooses $x_t = b$ with $b \ge a$, then the entire interval $[a, b]$ is $\cF_t$-feasible (i.e., $[a, b] \subseteq \mathcal{F}_t$).
\end{lemma}

\begin{proof}
    With $a=z_t$, $\cC_t(a)= \{x\in \cF_t : [x,a]\subseteq \cF_t \ \text{or} \ [a,x]\subseteq \cF_t\}$ and hence an interval $[L, R]$ containing $a$.
    

    If $b \in \mathcal{C}_t(a)$  then $b \in [L, R]$. 
    Finally, since $\mathcal{C}_t(a) \subseteq \mathcal{F}_t$, and if $b \ge a$, we have $[a, b] \subseteq \mathcal{F}_t$.
\end{proof}

Recall that $\cA_B$ chooses action $x_t = \arg\min_{x\in \cC_t(z_t)}f_t(x)$, and if the minimizer is not unique, it chooses the one that is closest to $x_{t-1}$.
\begin{lemma}[Global Monotonicity from Local Descent]\label{lem:global_monotonicity}
    Let $a = z_t$ and $b = x_t$.
    \begin{itemize}
        \item If $b > a$, then $f_t$ is monotonically non-increasing on the feasible half-line $(-\infty, b] \cap S_t$.
        \item If $b < a$, then $f_t$ is monotonically non-decreasing on the feasible half-line $[b, \infty) \cap S_t$.
    \end{itemize}
\end{lemma}

\begin{proof}
    Assume $b > a$. By Lemma \ref{lem:topology_1d}, $[a, b] \subseteq \mathcal{C}_t(a)$.
    Since $b$ minimizes $f_t$ on $\mathcal{C}_t(a)$, we must have $f_t(b) \le f_t(a)$.
    
    Since $f_t$ is convex, for any $x \in S_t$ with $x < a$:
    \[ \frac{f_t(a) - f_t(x)}{a - x} \le \frac{f_t(b) - f_t(a)}{b - a} \]
    The RHS is $\le 0$ because $f_t(b) \le f_t(a)$ and $b > a$.
    Thus $f_t(a) - f_t(x) \le 0 \implies f_t(x) \ge f_t(a) \ge f_t(b)$.
    
    This establishes monotonicity on $(-\infty, a]$. For the interval $[a, b]$, since $b$ is the minimizer of a convex function on $[a, b]$, $f_t$ must be non-increasing on $[a, b]$ as well.
    Combining these implies $f_t$ is non-increasing on $(-\infty, b] \cap S_t$. The other case follows symmetrically.
\end{proof}

\subsection{Competitive Analysis}

Let us denote $\cA_B$'s and $\opt$'s actions at time step $t$ by $x_t$ and $y_t$, respectively. To show competitiveness, we will be using a potential function $\Phi_t = |x_t-y_t|$ that captures the distance between $\cA_B$'s and $\opt$'s actions. We also denote the change in potential at timestep $t$ as $\Delta\Phi_t = \Phi_t- \Phi_{t-1}$. Note that $\Phi_t\ge0$ for all $t$ and $\Phi_0 = 0$ as $x_0 = y_0$.

\begin{proof}[Proof of Theorem \ref{thm:cr5}]
    We analyze algorithm $\cA_B$ by breaking down each time step $t$ into two virtual steps with the following definitions to efficiently count the movement cost $\cA_B$ by first moving to $z_t$ from $x_{t-1}$ and then from $z_t$ to $x_t$.
        \begin{center}
    \begin{tabular}{|p{0.45\textwidth}|p{0.45\textwidth}|}
        \hline
        \textbf{Virtual Step 1:} & \textbf{Virtual Step 2:} \\
        \hline
        \vspace{-0.5em}
        \begin{itemize}
            \item Cost function: $f_t^1 \equiv 0$
            \item Constraint set: $S_t^1 = S_t$
            \item $\mathcal{A}$'s action: $x_t^1 = z_t \in S_t$
            \item $\opt$'s action: $y_t^1 \in S_t$
            \item $\Phi_t^1 = |x_t^1-y_t^1|$
            \item $\Delta\Phi_t^1 = \Phi_t^1 - \Phi_{t-1}^2$
        \end{itemize}
        &
        \vspace{-0.5em}
        \begin{itemize}
            \item Cost function: $f_t^2 = f_t$
            \item Constraint set: $S_t^2 = S_t$
            \item $\mathcal{A}$'s action: $x_t^2 = x_t \in S_t$
            \item $\opt$'s action: $y_t^2 = y_t\in S_t$
             \item $\Phi_t^2 = |x_t^2-y_t^2|$
            \item $\Delta\Phi_t^2 = \Phi_t^2 - \Phi_{t}^1$
        \end{itemize} \\
        \hline
    \end{tabular}
\end{center}   
Let $\textsf{Lin}_{\cA_B}^i(t)$ and $\textsf{Lin}_{\opt}^i(t)$ be the cost of $\cA_B$ and $\opt$ in the $i^{th}, i=1,2$ virtual step at time $t$. 
Then, $$\textsf{Lin}_{\cA_B}^1(t) = f_t^1(z_t)+ |z_t-x_{t-1}| = 0 + |z_t-x_{t-1}|,$$ and 
$$\textsf{Lin}_{\cA_B}^2(t) = f_t^2(x_t)+ |x_t-z_t| = f_t(x_t) + |x_t-z_{t}|,$$ which results in the total cost incurred by $\cA_B$ across the two virtual steps as
     \begin{equation}\label{eq:costABub}
\textsf{Lin}_{\cA_B}^1(t)+\textsf{Lin}_{\cA_B}^2(t)=|z_t-x_{t-1}|+(f_t(x_t)+|x_t-z_t|) \ge (f_t(x_t)+|x_t-x_{t-1}|) = \textsf{Lin}_{\cA_B}(t).
\end{equation}
    Similarly, 
    $$\textsf{Lin}_\opt^1(t) = f_t^1(y_t^1)+ |y_t^1-y_{t-1}| = 0 +|y_t^1-y_{t-1}|,$$ and 
$$\textsf{Lin}_\opt^2(t) = f_t^2(y_t^2)+ |y_t^2-y_t^1| = f_t(y_t)+ |y_t-y_t^1|.$$

    If $y_t^1 = y_t$, it is clear that $$\textsf{Lin}_\opt^1(t) + \textsf{Lin}_\opt^2(t) = \textsf{Lin}_\opt(t) = f_t(y_t)+ |y_t-y_{t-1}|.$$ 
    Thus, if $y_t^1 \in S_t$ is defined to be the choice made to minimize $\textsf{Lin}_\opt^1(t) + \textsf{Lin}_\opt^2(t)$, then 
    \begin{equation}\label{eq:costoptlb}
    \textsf{Lin}_\opt^1(t) + \textsf{Lin}_\opt^2(t) \le \textsf{Lin}_\opt(t).
    \end{equation}
    \begin{lemma}\label{lem:1d_per_step}
        At each timestep $t$, 
        \[
        \sum_{i=1,2} \textsf{Lin}_{\cA_B}^i(t) + 3 \sum_{i=1,2} \Delta\Phi_t^i \le 5\sum_{i=1,2}\textsf{Lin}_{\opt}^i(t).
        \]
    \end{lemma}
\begin{proof}
    Using the triangle inequality, the change in potential at a virtual step $t^i$ ($i =1$ or $2$) can be expressed as,
    \[
    \Delta\Phi_t^1 = |x_t^1-y_t^1| - |x_{t-1}^2-y_{t-1}^2| \le |x_t^1-y_t^1| - ( |x_{t-1}^2-y_{t}^1| - |y_{t}^1-y_{t-1}^2|),
    \]
    \[
    \Delta\Phi_t^2 = |x_t^2-y_t^2| - |x_{t}^1-y_{t}^1| \le |x_t^2-y_t^2| - ( |x_{t}^1-y_{t}^2| - |y_{t}^2-y_{t}^1|).
    \]
    where the last term in both expressions can be counted towards the movement cost of $\opt$ in the virtual step $t^i$ using \eqref{eq:costoptlb} and noting that $y_{t}^2=y_t$ for all $t$. Let $\Tilde{\Delta}\Phi_t^1 = |x_t^1-y_t^1| - |x_{t-1}^2-y_{t}^1|$, and $\Tilde{\Delta}\Phi_t^2 = |x_t^2-y_t^2| - |x_{t}^1-y_{t}^2|$. 
    Therefore, in light of \eqref{eq:costABub} and \eqref{eq:costoptlb}, to prove Lemma \ref{lem:1d_per_step} it is enough to show that
    \[
    \sum_{i=1,2}\textsf{Lin}_{\cA_B}^i(t) + 3\sum_{i=1,2}\Tilde{\Delta}\Phi_t^i \le 5f_t(y_t).
    \]

    \noindent\textbf{Virtual Step 1 Analysis:}\\
    If $x_{t-1} \notin S_t$, the move to $z_t$ costs $|x_{t-1} - z_t|$.
    Since $y^1_t \in S_t$:
    \[ \Tilde{\Delta}\Phi_t^1 = |z_{t} - y^1_t| - |x_{t-1} - y^1_t| = - |x_{t-1} - z_t| \]
    Therefore, $\textsf{Lin}^1_{\cA_B}(t) + 3\Tilde{\Delta}\Phi_t^1  = f_t^1(z_t)+|x_{t-1} - z_t| -3|x_{t-1} - z_t| = 0 -2|x_{t-1} - z_t|\le 0$.\\
    If $x_{t-1}\in S_t$, $z_t = x_{t-1}$ and hence, $\textsf{Lin}_{\cA_B}^1(t) = f_t^1(z_t)+ \Tilde{\Delta}\Phi_t^1 = 0$.\\
    
    \noindent\textbf{Virtual Step 2 Analysis:}
    We assume $x_t \ge z_t$, but similar arguments work for the case when $x_t<z_t$. Recall that $y_t^2=y_t, x_t^2=x_t$ and $x_t^1=z_t$ for all $t$.
    \paragraph{Case 1: $\opt$ is ``ahead'' ($z_t \le x_t \le y_t$, i.e, $\cA_B$ moves towards $\opt$)}
    \begin{align}\nn
        \Tilde{\Delta} \Phi_t^2 &= |x_t-y_t| - |z_t-y_t| = -|z_t-x_t|\\ \label{eq:dummyAB1}
        \implies \textsf{Lin}_{\cA_B}^2(t) + 3\Tilde{\Delta} \Phi_t^2&= f_t(x_t) + |x_t - z_t| - 3|x_t - z_t| = f_t(x_t) - 2|x_t - z_t|
    \end{align} 
    \begin{itemize}
        \item If $x_t$ is on the boundary of $\cC_t(z_t)$, the constraint is active, i.e, $ f_t(x_t) = |x_t-z_t|$. Therefore,
        $\textsf{Lin}_{\cA_B}^2(t) + 3\Tilde{\Delta} \Phi_t^2 = -|x_t-z_t| \le 0$.
        \item If $x_t$ is in the interior of $\cC_t(z_t)$, the function $f_t$ must be non-decreasing in $[L, x_t]$ and $[x_t,R]$ (where $L$ and $R$ are the boundary points of the interval $\cC_t(z_t)$). By convexity, this implies that $x_t$ is the unconstrained local minimizer of $f_t$. Therefore, $f_t(x_t) \le f_t(y_t)$ which using \eqref{eq:dummyAB1} implies 
        $$\textsf{Lin}_{\cA_B}^2(t) + 3\Tilde{\Delta} \Phi_t^2 = f_t(x_t)- 2|x_t-z_t| \le f_t(x_t) \le 5 f_t(y_t).$$
    \end{itemize}

    \paragraph{Case 2: $\opt$ is ``in-between'' ($z_t \le y_t < x_t$, i.e,
    $\cA_B$ overshoots $\opt$)}
    In this case, we use the following trivial upper bound 
        \[ \Tilde{\Delta} \Phi_t^2 = |x_t - y_t| - |z_t - y_t| \le |x_t - z_t|. \]
Using $\cF_t$-feasibility of $x_t$ ($|x_t - z_t| \le f_t(x_t)$), we get
    \[ \textsf{Lin}_{\cA_B}^2(t) + 3\Tilde{\Delta}\Phi_t^2 \le f_t(x_t)+|x_t - z_t| + 3|x_t - z_t| \le 5 f_t(x_t).\]
    Invoking the fact that  $y_t \in [z_t, x_t]$, Lemma \ref{lem:topology_1d} implies $y_t \in \cC_t(z_t)$.
    Moreover, $x_t$ minimizes $f_t$ on $\cC_t(z_t)$. Thus, $f_t(x_t) \le f_t(y_t)$, using which we get 
        \[ \textsf{Lin}_{\cA_B}^2(t) + 3\Tilde{\Delta}\Phi_t^2  \le 5 f_t(y_t).\]

    \paragraph{Case 3: $\opt$ is ``behind'' ($y_t < z_t \le x_t$, i.e, $\cA_B$ moves away from $\opt$.)} This can occur when constraint sets are revealed such that $x_{t-1}$ is in the interior of $S_t$, which would imply $z_t=x_{t-1}$. In this case,
    \[ \Tilde{\Delta} \Phi_t^2 = |x_t - z_t| \]
    From Lemma \ref{lem:global_monotonicity}, $f_t(y_t) \ge f_t(x_t)$. From feasibility, $f_t(x_t) \ge |x_t - z_t|$.
    \[ \textsf{Lin}_{\cA_B}^2(t) + 3\Tilde{\Delta}\Phi_t^2  = f_t(x_t)+|x_t - z_t| + 3 |x_t - z_t| \le 5 f_t(x_t) \le 5 f_t(y_t)\]
    This concludes the proof of the lemma.
\end{proof}
    Summing Lemma \ref{lem:1d_per_step}'s expression across all timesteps $T$ while using telescoping series and noting that $\Phi_T^2\ge 0$, we get:
    \begin{align*}
        \sum_{t=1}^T\sum_{i=1,2}\textsf{Lin}_{\cA_B}^i(t) + 3\Delta\Phi_t^i&\le 5\sum_{t=1}^T\sum_{i=1,2}\textsf{Lin}_\opt^i(t),\\
        \stackrel{\eqref{eq:costABub} \ and \ \eqref{eq:costoptlb}}\implies \sum_{t=1}^T\textsf{Lin}_{\cA_B}(t) &\le 5\sum_{t=1}^T\textsf{Lin}_\opt(t),
    \end{align*}
    which proves the $5$-competitiveness of the proposed algorithm.
\end{proof}

\section{Proof of Theorem \ref{thm:sharp_comp}}\label{app:sharp_comp}

For any algorithm $\cA$, compared to \eqref{eq:lincombcost}, let the surrogate cost $\widehat{\textsf{Lin}}_\cG(T)$ be defined as 
\begin{equation}\label{eq:lincombcostsurrogate}
\widehat{\textsf{Lin}}_\cG(T) = \textsf{Lin}_\cA(T) - \sum_{t=1}^T v_t,
\end{equation}
where $v_t= \min_{x\in S_t} f_t(x)$. Recall that $\bX^* = \arg \min_{x\in S_t} f_t(x)$. Let $\cG$ be the greedy algorithm that always chooses $x_t\in \bX^*$ that is closest to $x_{t-1}$, starting from some point $x_0$. For simplicity of exposition, we let the actions of $\cG$ be represented by $x_t = x_t^\star \in \bX^*$. Then it follows that 
 $$\textsf{Lin}_\cG(T) \ge \textsf{Lin}_\opt(T)>\sum_{t=1}^T v_t\ge 0,$$ and we have
\begin{align*}
    \left(\textsf{Lin}_\cG(T) - \sum_{t=1}^T v_t\right)\textsf{Lin}_\opt(T) &\ge \left(\textsf{Lin}_\opt(T) - \sum_{t=1}^T v_t\right)\textsf{Lin}_\cG(T),\\
    \implies \frac{\widehat{\textsf{Lin}}_\cG(T)}{\widehat{\textsf{Lin}}_\opt(T)} &\ge \frac{\textsf{Lin}_\cG(T)}{\textsf{Lin}_\opt(T)}.
\end{align*}
Therefore, for deriving an upper bound on the competitive ratio for solving \eqref{eq:lincombcost} it is sufficient to derive an upper bound on the surrogate costs, which will accomplish in the following.

We next prove a structural lemma that is a modified version of a similar result from  \cite{pmlr-v125-argue20a}.
\begin{lemma}\label{lem:structure}
    For any three points $x,y,z \in \R^d$, at least one of the following is true,
    \begin{enumerate}
        \item $\Vert y-z\Vert - \Vert y-x\Vert \leq -\frac{1}{2}\Vert x-z\Vert.$
        \vspace{1mm}
        \item $\Vert y-z\Vert \geq \frac{1}{4}\Vert x-z\Vert.$
    \end{enumerate}
\end{lemma}
\begin{proof}
    Let us assume for some $x,y,z\in\mathbb{R}^n$ the second condition is not true, i.e.,
    \begin{equation}\label{eqn:contradiction}
        \Vert y-z\Vert < \frac{1}{4}\Vert x-z\Vert.
    \end{equation}
    From the triangle inequality, we also have,
    $$\Vert y-x\Vert \ge \Vert x-z\Vert -\Vert y-z\Vert.$$
    Using  \eqref{eqn:contradiction}, we get
    \begin{align*}
        \Vert y-x\Vert &\ge \Vert x-z\Vert - \frac{1}{4}\Vert x-z\Vert,\\
        &= \frac{1}{2}\Vert x-z\Vert + \frac{1}{4}\Vert x-z\Vert,\\
        &\ge \frac{1}{2}\Vert x-z\Vert + \Vert y-z\Vert,\\
        \implies \Vert y-z\Vert - &\Vert y-x\Vert\leq-\frac{1}{2}\Vert x-z\Vert.
    \end{align*}
    Therefore, at least one of the two conditions holds for any choice of $x,y,$ and $z$, concluding the proof.
\end{proof}
\begin{proof}[Proof of Theorem \ref{thm:sharp_comp}]
We consider a potential function $\Phi_t$ that at timestep $t$ captures the distance between $\cG$'s action $x_t$ and $\opt$'s action $y_t$, defined by
$$\Phi_t = \Vert x_t-y_t\Vert.$$
Note that $\Phi_0 = 0$ and $\Phi_t\ge 0$. We also define $\Delta\Phi_t :=  \Phi_t - \Phi_{t-1}$ as the change in potential at time $t$.

    We first use the triangle inequality to bound the change in potential as:
    \begin{align*}
    \Delta \Phi_t = \Vert x_t-y_t\Vert-\Vert x_{t-1}-y_{t-1}\Vert \le \Vert x_t-y_t\Vert- (\Vert x_{t-1}-y_t\Vert -\Vert y_t-y_{t-1}\Vert).
    \end{align*}
    Since $\Vert y_t - y_{t-1}\Vert$ accounts for the movement cost part of $\widehat{\textsf{Lin}}_\opt(t)$, to prove Theorem \ref{thm:sharp_comp} it is enough to show,
    $$\widehat{\textsf{Lin}}_\cG(t) +2\cdot\Tilde{\Delta}\Phi_t \le \frac{12}{\alpha}\cdot (f_t(y_t) - v_t) = \frac{12}{\alpha}\widehat{\textsf{Lin}}_\opt(t),$$
    where $\Tilde{\Delta}\Phi_t := \Vert x_t-y_t\Vert-\Vert x_{t-1}-y_t\Vert$, which captures the change in potential due to $\cG$'s actions. Using Lemma \ref{lem:structure} with $x = x_{t-1}, y=y_t,$ and $z=x_t$, we know that either of the following holds:
    \begin{enumerate}
        \item $\Tilde{\Delta}\Phi_t = \Vert x_t-y_t\Vert - \Vert x_{t-1}-y_t\Vert \le -\frac{1}{2}\Vert x_{t-1}-x_{t}\Vert,$ \ \text{or}
        \item $\Vert y_t-x_t\Vert\ge\frac{1}{4}\Vert x_{t-1}-x_{t}\Vert.$
    \end{enumerate}
    \paragraph{When the first condition is true, i.e.,} $\Tilde{\Delta}\Phi_t \le -\frac{1}{2}\Vert x_{t-1}-x_{t}\Vert$. Note that $\widehat{\textsf{Lin}}_\cG(t) = \Vert x_t-x_{t-1}\Vert$ and hence
    \begin{align*}
        \widehat{\textsf{Lin}}_\cG(t) + 2\cdot \Tilde{\Delta}\Phi_t&\le \Vert x_t-x_{t-1}\Vert-\Vert x_t-x_{t-1}\Vert,\\
        &=0\le \frac{12}{\alpha}\cdot(f_t(y_t)-v_t).
    \end{align*}
    \paragraph{When the second condition is true i.e., } $\Vert y_t-x_t\Vert\ge\frac{1}{4}\Vert x_{t-1}-x_{t}\Vert$. Since $x_t=x_t^\star$ for $\cG$ and $y_t\in S_t$, we use $f_t$'s $\alpha$-sharpness with unique global minimizer $x_t^\star$ for each $S_t$, $\cG$ is $\frac{12}{\alpha}$-competitive for \eqref{eq:lincombcost} with dynamic constraints.
as follows:
    \begin{align}\label{eq:greedydummy1}
        f_t(y_t) - v_t \ge \alpha \Vert y_t-x_t^\star\Vert \ge\frac{\alpha}{4}\Vert x^\star_t-x^\star_{t-1}\Vert=\frac{\alpha}{4}\widehat{\textsf{Lin}}_\cG(t).
    \end{align}
    By the triangle inequality,
    \begin{equation}\label{eq:greedydummy2}
\Tilde{\Delta}\Phi_t = \Vert x_t^\star-y_t\Vert-\Vert x_{t-1}^\star-y_t\Vert\le \Vert x_{t}^\star-x_{t-1}^\star\Vert=\widehat{\textsf{Lin}}_\cG(t).
\end{equation}

    Combining \eqref{eq:greedydummy1} and \eqref{eq:greedydummy2}, we have,
    \begin{align*}
        \widehat{\textsf{Lin}}_\cG(t) +2\cdot\Tilde{\Delta}\Phi_t \le 3\Vert x_{t}^\star-x_{t-1}^\star\Vert\le \frac{12}{\alpha}\cdot(f_t(y_t) - v_t) = \frac{12}{\alpha}\widehat{\textsf{Lin}}_\opt(t).
    \end{align*}
    This completes the proof for the bound on surrogate cost. Summing from $t=1$ to $T$ and using $\Phi_T\ge 0=\Phi_0$ proves the competitiveness. 
\end{proof}

\section{Simulations}\label{sec:sim}
In this section, we present some numerical results to compare the regret and movement cost of the three algorithms, i.e. \textsc{Frugal}, Greedy, and LSP ($\epsilon$) through simulations. To clearly show the difference between these algorithms, we consider the input instance that we used to prove Theorem~\ref{thm:lbGreedy}, i.e. the lower bound on the movement cost of the Greedy algorithm when $f$ is strongly convex. Recall that the cost function used in proof of Theorem~\ref{thm:lbGreedy} is $f(x, y) = x^2 + y^2$, which is strongly convex and smooth, and that the set $\mathcal{X}$ is a square of side $\frac{D}{\sqrt{2}}$ which is at a distance $r_0$ from the origin. The sets $S_1, S_2, \dots, S_T$ are constructed such that each minimizer $x_t^\star$ is on either side of a rectangle of width $k$ and length $\frac{D}{\sqrt{2}}$. We fix $r_0 = 1$, $D = 4$ and $\displaystyle k = \sqrt{\frac{\frac{D}{\sqrt{2}} r_0}{\frac T2 + 1}}$ throughout, as in \eqref{eq:k_value_lb}. The value of $T$ will be chosen later for the different simulations.

\subsection{Trajectory visualization}

\begin{figure}[!ht]
    \centering
    \includegraphics[width=\linewidth]{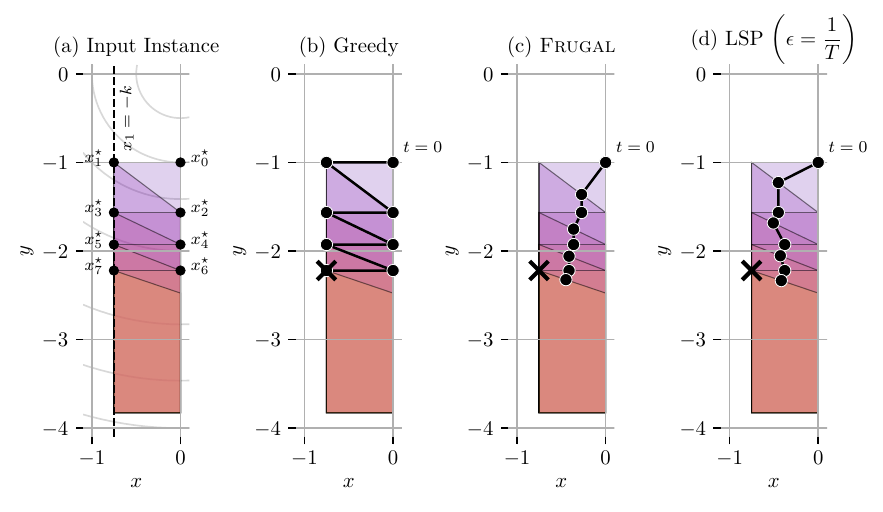}
    \caption{$(a)$ specifies the input via describing sets $S_1, \dots, S_T$, along with the respective minimizers $x_i^\star$ and the contours of the cost function $f(x) = x^2 + y^2$. $(b)$, $(c)$, and $(d)$ are respectively the trajectories of the Greedy, \textsc{Frugal}, and LSP ($\epsilon = 1/T$) algorithms. The point marked `X' is the static minimizer $x^\star_T$.}
    \label{fig:trajectories}
\end{figure}

First, we visualize the trajectories of the three algorithms for $T = 7$ in Figure~\ref{fig:trajectories}. For this particular input instance, \textsc{Frugal} always projects and never {\it jumps}, while Greedy by definition moves along optimizers. The LSP algorithm ($\epsilon = 1/T$) follows a trajectory that is intermediate between these two trajectories.

\subsection{Total Regret and Movement Cost vs T}

Next, we plot the the regret  and  movement cost for  Greedy, \textsc{Frugal}, and LSP against $T$ in Figure~\ref{fig:regrets_mc}. Each value of $T$ requires a distinct input instance, since the sets $S_1, \dots, S_T$ are picked according to the value of $T$ as described in the proof of Theorem~\ref{thm:lbGreedy}. 

\begin{figure}[!ht]
    \centering
    \includegraphics[width=\linewidth]{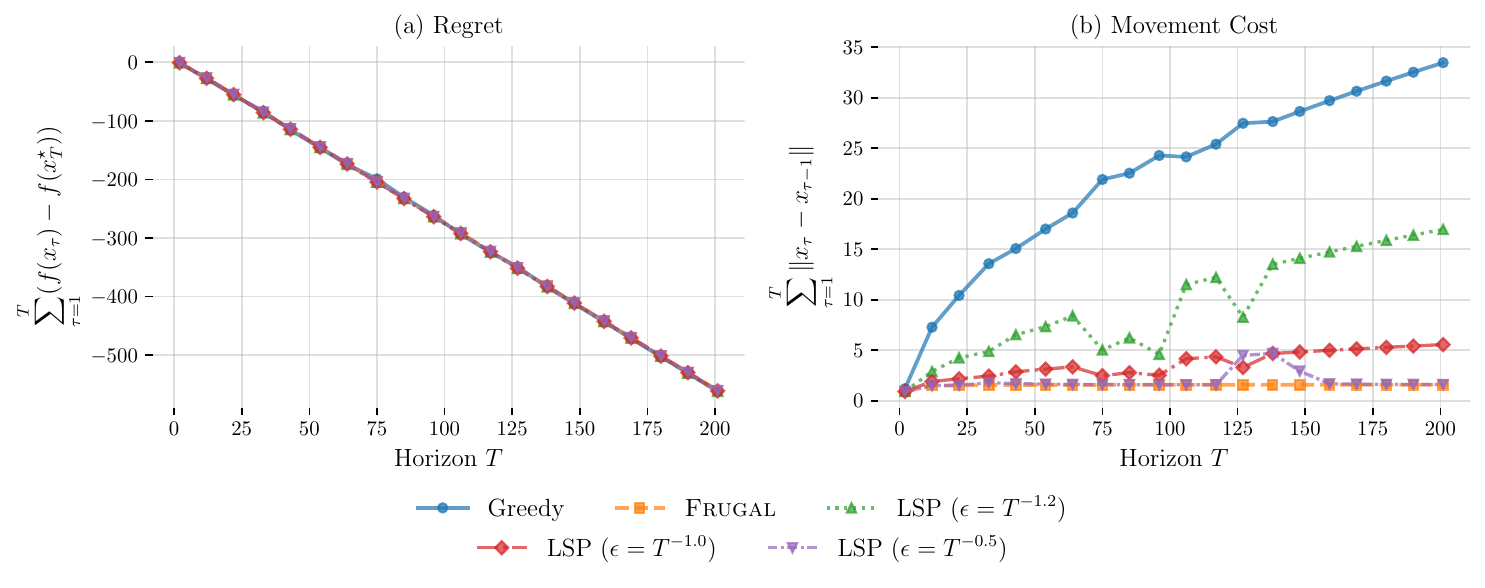}
    \caption{Regret and movement cost for the Greedy, \textsc{Frugal}, and LSP algorithms for $T$ varying linearly from $1$ to $200$.}
    \label{fig:regrets_mc}
\end{figure}

We observe that the regret for all three of the algorithms remains highly negative and that there is no significant difference between the regret of the three algorithms. The movement cost of the \textsc{Frugal} algorithm is the smallest and similar to LSP with the largest $\epsilon$ value of $\epsilon = T^{-0.5}$, while LSP algorithm with decreasing $\epsilon$ has greater movement cost. The Greedy algorithm has a movement cost of $\Omega(\sqrt{T})$ for this family of lower-bound instances in the simulations, in line with Theorem~\ref{thm:lbGreedy}.

\subsection{Cumulative Regret and Movement Cost vs $t$}

Finally, we plot the cumulative regret $\sum_{\tau = 1}^t (f(x_\tau) - f(x_t^\star))$ and cumulative movement cost $\sum_{\tau = 1}^t ||x_{\tau} - x_{\tau - 1}||$ for a fixed value of $T$ and for varying values of $t$ in Figure~\ref{fig:regret_movement_cost_small_T}. 
Thus, there is only one input instance throughout and the regret and the movement cost incurred are just due to the first $t$ actions. 

The input instance is constructed as follows. First, we generate the nested sets \(S_1,\dots,S_{T_{\mathrm{freeze}}}\) exactly as in the lower-bound instance of Theorem~\ref{thm:lbGreedy}, with the time horizon parameter set to \(T_{\mathrm{freeze}}\). We then freeze the sequence by repeatedly presenting the set $S_{T_{\mathrm{freeze}}}$, i.e. 
$S_t=S_{T_{\mathrm{freeze}}}$ for $T_{\mathrm{freeze}}\le t \le T$.
Such an input instance illustrates the nature of the {\it jump} action played by the \textsc{Frugal} algorithm and the long-term regret suffered by the three algorithms.

\begin{figure}[!ht]
    \centering
    \includegraphics[width=\linewidth]{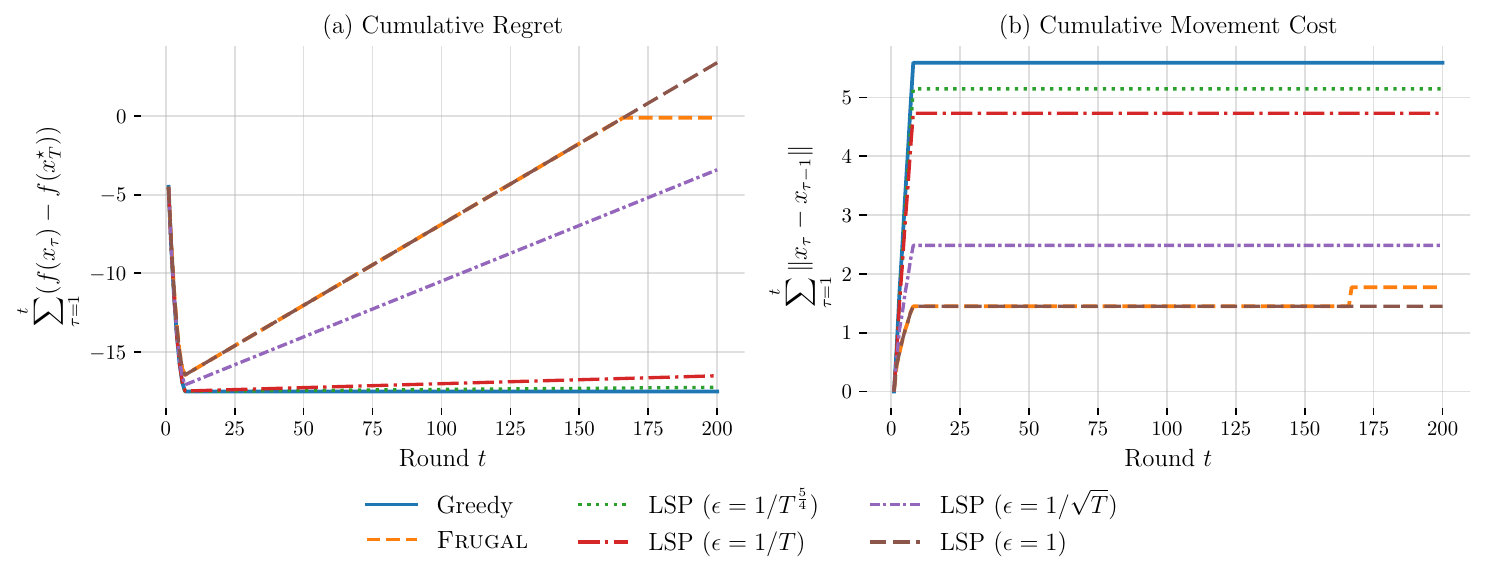}
    \caption{$(a)$ Cumulative regret and $(b)$ cumulative movement cost for various algorithms. }
    \label{fig:regret_movement_cost_small_T}
\end{figure}

We take $T_{\text{freeze}}=7$ and $T = 200$. From $t = 1$ to $t = 7$, all the algorithms incur similar negative regret and a movement cost which is the smallest for the \textsc{Frugal} algorithm and the largest for the Greedy algorithm. After the set is frozen at $t=7$, the Greedy algorithm incurs no further regret or movement cost. The action $x_7$ played at time $7$  and the optimal action $x_T^\star$ are different for both \textsc{Frugal} and LSP algorithms, and thus they incur positive regret from $t = 8$ to $t = 165$. At $t = 166$, the \textsc{Frugal} algorithm {\it jumps} to the optimizer of $f$ over set $S_7$, ensuring $0$ cumulative regret for $t > 166$. In contrast, the LSP algorithm continues to incur positive regret since it accepts positive regret.





\end{document}